\documentclass{article}
\pdfoutput=1
\usepackage{pdfpages}
\usepackage[nonatbib,final]{neurips_2022}
\usepackage[numbers]{natbib}
\usepackage[utf8]{inputenc} 
\usepackage[T1]{fontenc}    
\usepackage{hyperref}       
\usepackage{url}            
\usepackage{booktabs} 
\usepackage{amsfonts}       
\usepackage{nicefrac}       
\usepackage{microtype}      

\usepackage{mdframed}
\usepackage{amsmath,mathrsfs,amssymb,mathtools,bm}
\usepackage{scalerel}
\usepackage{amsthm}
\usepackage{txfonts}
\usepackage{fontawesome}
\usepackage{wrapfig}
\usepackage{todonotes}
\usepackage{xcolor}
\usepackage{setspace}
\usepackage{multirow}
\usepackage{bm,bbm}
\usepackage{paralist}
\usepackage{enumitem}
\usepackage{caption}[small]
\usepackage{subcaption} 

\newtheorem{corollary}{Corollary}
\newtheorem{problem*}{Problem}

\newtheorem{theorem}{Theorem}

\newtheorem{proposition}{Proposition}

\newtheorem*{example*}{Example}

\usepackage{todonotes}
\usepackage{xcolor}
\usepackage{setspace}
\usepackage{multirow}
\usepackage{bm,bbm}
\usepackage{paralist}

\DeclareMathOperator*{\argmin}{argmin}

\usepackage{amsmath}
\usepackage{multirow}
\usepackage{multicol}
\usepackage{url}

\definecolor{darkgreen}{RGB}{204,102,0}

\newcommand{\cA}{\mathcal{A}}  
 \newcommand{\cD}{\mathcal{D}}

 \newcommand{\cL}{\mathcal{L}}
 
\newcommand{\cO}{\mathcal{O}}

 \newcommand{\cY}{\mathcal{Y}}
 \newcommand{\cX}{\mathcal{X}}

 \newcommand{\RR}{\mathbb{R}}

\newcommand{\btheta}{{\bm{\theta}}}

\newcommand{\minimize}[1]{\underset{{#1}}{\text{minimize}}}

\usepackage{esvect}

\makeatletter
\newcommand{\oset}[3][0ex]{\mathrel{\mathop{#3}\limits^{\vbox to#1{\kern-2\ex@\hbox{$\scriptstyle#2$}\vss}}}}
\makeatother
\newcommand{\optimal}[1]{\oset{\scalebox{.5}{$\star$}}{#1}}

\usepackage{letltxmacro}
\LetLtxMacro\orgvdots\vdots
\LetLtxMacro\orgddots\ddots

\usepackage{neurips_2022}

\usepackage[utf8]{inputenc} 
\usepackage[T1]{fontenc}    
\usepackage{hyperref}       
\usepackage{url}            
\usepackage{booktabs}       
\usepackage{amsfonts}       
\usepackage{nicefrac}       
\usepackage{microtype}      
\usepackage{xcolor}         

\def\opttheta{\optimal{\bm{\theta}}}
\def\bartheta{\bar{\bm{\theta}}}

\title{Pruning has a disparate impact on model accuracy}

\author{%
  Cuong Tran\\
  Department of Computer Science\\
  Syracuse University\\
  \texttt{cutran@syr.edu} \\
  \And
  Ferdinando Fioretto\\
Department of Computer Science\\
Syracuse University\\
  \texttt{ffiorett@syr.edu} \\
  \AND
  Jung-Eun Kim\thanks{This work was partly conducted when Jung-Eun Kim was an assistant professor at Syracuse University.} \\
  Department of Computer Science\\
  North Carolina State University \\
  \texttt{jung-eun.kim@ncsu.edu} \\
  \And
  Rakshit Naidu\\
  Department of Computer Science\\
  Carnegie Mellon University\\
  \texttt{rnemakal@andrew.cmu.edu} \\
}

\begin{document}

\maketitle

\begin{abstract}
Network pruning is a widely-used compression technique that is able to significantly scale down overparameterized models with minimal loss of accuracy. This paper shows that pruning may create or exacerbate disparate impacts. The paper sheds light on the factors to cause such disparities, suggesting differences in gradient norms and distance to decision boundary across groups to be responsible for this critical issue. It analyzes these factors in detail, providing both theoretical and empirical support, and proposes a simple, yet effective, solution that mitigates the disparate impacts caused by pruning. 
\end{abstract}

\section{Introduction}
\label{sec:intro}

As deep learning models evolve and become more powerful, they also
become larger and more costly to store and execute. The trend hinders
their deployment in resource-constrained platforms, such as embedded
systems or edge devices, which require efficient models in time and
space.  To address this challenge, studies have developed a variety
of techniques to prune the relatively insignificant or insensitive
parameters from a neural network while ensuring competitive
accuracy \cite{aghli2021combining,
baykal2019sipping,
blalock2020state,
Renda2020Comparing,
Han2015NIPS,sehwag2019compact,
zhang2018systematic}. 
When a model needs to be developed to 
fit given and certain requirements in size and resource consumption, 
a pruned model which is derived from a large, rigorously-trained, and
(often) over-parameterized model, is regarded as a de-facto standard. 
That is because it performs incomparably better than a same-size dense 
model which is trained from scratch, when the same amount of
effort and resources are invested.

In spite of its strengths, pruning has been showed to induce or 
exacerbate disparate effects in the accuracy of the resulting 
reduced models \cite{Hooker2020CharacterisingBI, Hooker2020WhatDC}. 
Intuitively, the removal of model weights affects the process 
in which the network separates different classes, which can have 
contrasting consequences for different groups of individuals.
This paper further shows that the accuracy of the pruned models tends 
to increase (decrease) more in classes that had already 
high (low) accuracy in the original model, leading to a  
``the rich get richer'' and ``the poor get poorer'' effect. 
This \emph{Matthew} effect is illustrated in Figure \ref{fig:motivation}. 
The figure shows the accuracy of a facial recognition task on 
different demographic groups for several pruning rates (indicating 
the percentage of parameters removed from the original models).
Notice how the accuracy of the majority group (White) tends to 
increase while that of the minority groups tends to decrease as the 
pruning ratio increases.

\begin{figure*}[t]
 \centering
    \includegraphics[width=\linewidth,height=100pt]{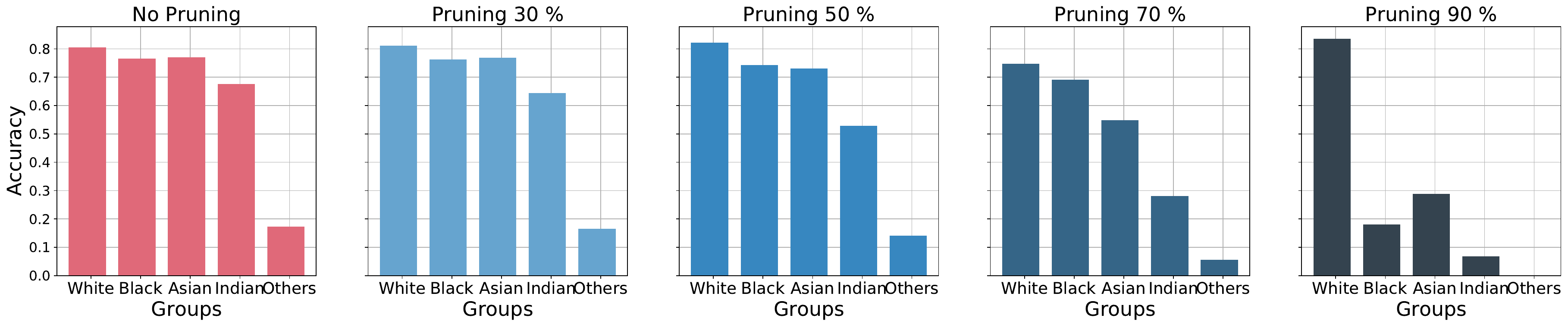}
    \caption{Accuracy of each demographic group in the UTK-Face dataset using Resnet18 \cite{he2016deep}, at the increasing of the pruning rate.}
   \label{fig:motivation}
\end{figure*}

Following these observations, we shed light on the factors to 
cause such disparities. The theoretical findings suggest the presence 
of two key factors responsible for why accuracy disparities arise 
in pruned models: 
{\bf(1)} disparity in \emph{gradient norms} across groups, and {\bf
(2)} disparity in \emph{Hessian matrices} associated with the loss function computed using a group's data. 
Informally, the former carries information about the groups' local optimality, 
while the latter relates to model separability.
We analyze these factors in detail, providing both theoretical and
empirical support on a variety of settings, networks, and datasets. 
By recognizing these factors, we also develop a simple yet
effective training technique that largely mitigates the disparate impacts caused by pruning. The method is based on an alteration of the loss 
function to include components that penalize disparity of the average 
gradient norms and distance to decision boundary across groups. 

These findings are significant: {\em Pruning is a key enabler for
neural network models in embedded systems with deployments in
security cameras and sensors for autonomous devices for applications
where fairness is an essential need. 
Without careful consideration of the fairness impact of these
techniques, the resulting models can have profound effects on our
society and economy}.

\subsubsection*{Related work} 
Fairness and network pruning have been long studied in isolation. 
The reader is referred to the related papers and surveys on fairness 
\citep{
barocas2017fairness,caton2020fairness,dwork2012fairness,NIPS2016_9d268236,mehrabi2021survey} and 
pruning \citep{
aghli2021combining,baykal2019sipping,blalock2020state,
Renda2020Comparing,Han2015NIPS,sehwag2019compact,Wiebke2022,zhang2018systematic} 
for a review on these areas. 

The recent interest in assessing societal values of machine learning models 
has seen an increase of studies at the intersection of different properties 
of a learning model and their effects on fairness. For example, \citet{xu2021robust}
studies the setting of adversarial robustness and show that adversarial training 
introduces unfair outcomes in term of accuracy parity \cite{zhao2019inherent}. 
\citet{zhu2021rich} show that semisupervised settings can introduce unfair outcomes 
in the resulting accuracy of the learned models. 
Finally, several authors have also shown that private training can have 
unintended disparate impacts to the resulting models' outputs \cite{NEURIPS2019_eugene,fioretto:arxiv22b,
Fioretto:NeurIPS21b,Uniyal2021DPSGDVP,Zhu:AAAI2021} 
and downstream decisions \cite{pujol:20,Tran:IJCAI21}. 

Network compression has also been shown to have a profound impact towards 
the model fairness. For example, several works observed empirically that 
network compression may amplify unfairness in different learning tasks~\cite{paganini2020prune, 
Hooker2020WhatDC, Hooker2020CharacterisingBI, Joseph2020GoingBC}.
Most of the focus has been on vision tasks and in identifying the 
set of \emph{Pruning Identified Exemplars} (PIEs), the samples that are 
impacted most under the compression scheme and conclude that PIEs belongs 
to low frequency groups (those observed at the tail of the data distribution).
\citet{Blakeney2021SimonSE} further investigate how bias could be evaluated 
and mitigated in pruned neural networks using knowledge distillation while  
\citet{hosseinilearning} observed empirically that knowledge distillation 
processes may produce unfair student models. 
The impact of network compression towards fairness has also been assessed 
in natural language tasks. For example, \citet{Du2021WhatDC} and 
\citet{xu-etal-2021-beyond} empirically measure the robustness of 
compressed large language models, while \citet{Ahia2021TheLD} look 
into how compression schemes affects data-limited regimes. Finally, 
\citet{Xu2022CanMC} investigate ways to improve fairness in generative 
language models by compressing them.  
We also note that, concurrently to this work, \citet{good:22} studied 
the relative distortion in recall for various classes. They show that pruning 
has a Matthews effect on the recall for various classes and 
propose an algorithm to attenuate such an effect.

This paper builds on this body of work and their important empirical 
observations and provides a step towards a deeper theoretical understanding 
of the fairness issues arising as a result of pruning. It derives conditions
and studies the causes of unfairness in the context of pruning as well 
as it introduces mitigating guidelines.

\section{Problem settings and goals}
The paper considers datasets $D$ consisting of $n$ datapoints 
$(\bm{x}_i, a_i, y_i)$, with $i \in [n]$, drawn i.i.d.~from an unknown
distribution $\Pi$. Therein, $\bm{x}_i \in \cX$ is a feature vector,
$a_i \in \cA$ with $\cA = [m]$ (for some finite $m$) is a demographic group 
attribute, and $y_i \in \cY$ is a class label. For example, consider the 
case of a face recognition task. The training example feature $\bm{x}_i$ may 
describe a headshot of an individual, the protected attribute $a_i$
may describe the individual's gender or ethnicity, and $y_i$ represents
the identity of the individual. 
The goal is to learn a predictor $f_\btheta : \cX \to \cY$, where $\btheta$
is a $k$-dimensional real-valued vector of parameters that minimizes 
the empirical risk function:
\begin{align}
\label{eq:erm}
\optimal{\btheta} = \argmin_\btheta J(\bm{\btheta}; D) = 
\frac{1}{n} \sum_{i=1}^n \ell(f_\btheta(\bm{x}_i), y_i),
\end{align}
where $\ell: \cY \times \cY \to \RR_+$ is a non-negative \emph{loss function} 
that measures the model quality. 

We focus on analyzing properties arising when extracting a small model $f_{\bar{\btheta}}$ with $\bar{\btheta} \subset\, \optimal{\btheta}$ of size $|\bar{\btheta}| = \bar{k} \ll k$.
Model $f_{\bar{\btheta}}$ is constructed by pruning the least important 
values or filters from vector $\optimal{\btheta}$ (i.e., those with smaller 
values in magnitude) according to a prescribed criterion, such as an $\ell_p$ 
norm \cite{NIPS1988_07e1cd7d, Han2015NIPS}.
The paper focuses on understanding the fairness impacts (as defined next) 
arising when pruning general classifiers, such as neural networks. 

\paragraph{Fairness}
The fairness analysis focuses on the notion of \emph{excessive loss}, 
defined as the difference between the original and the pruned risk 
functions over some group $a \in \cA$:
\begin{equation}
\label{eq:2}
    R(a) = J(\bar{\btheta}; D_a) - J(\optimal{\btheta}; D_a),
\end{equation}
where $D_a$ denotes the subset of the dataset $D$ containing samples 
($\bm{x}_i, a_i, y_i$) whose group membership $a_i = a$. 
Intuitively, the excessive loss represents the change in loss (and thus, in 
accuracy) that a given group experiences as a result of pruning.
Fairness is measured in terms of the maximal \emph{excessive loss difference}, 
also referred to as \emph{fairness violation}:
\begin{equation}
\label{eq:3}
    \xi(D) = \max_{a, a' \in \cA} |R(a) - R(a')|,
\end{equation}
defining the largest excessive loss difference across all protected groups. 
(Pure) fairness is achieved when $\xi(D) = 0$, and thus 
a fair pruning method aims at minimizing the excessive loss difference. 

The goal of this paper is to shed light on why fairness issues arise 
(i.e., $ R(a) > 0)$ as a result of pruning, why some groups suffer more than others (i.e., $R(a) > R(a'))$, and what mitigation measures could be taken
to minimize unfairness due to pruning.

We use the following notation: variables are denoted by calligraph 
symbols, vectors or matrices by bold symbols, and sets by uppercase symbols. 
Finally, $\| \cdot \|$ denotes the Euclidean norm and we use
$f_{\btheta}(\bm{x})$ to refer to the model' \emph{soft} outputs.
All proofs are reported in Appendix~\ref{sec:missing_proofs}.

\section{Fairness analysis in pruning: Roadmap}
\label{sec:pruning_impact}

To gain insights on how pruning may introduce unfairness, we start with providing a useful upper bound for a group's excessive loss. Its goal is to isolate key aspects of model pruning that are responsible for the observed unfairness. The following discussion assumes the loss function $\ell(\cdot)$ to be at least twice differentiable, which is the case for common ML loss functions, such as mean squared error or cross entropy loss.
\begin{theorem}
\label{thm:taylor} 
The \emph{excessive loss} of a group $a \in \cA$ is upper bounded by\footnote{
  With a slight abuse of notation, the results refer to $\bar{\btheta}$ as the homonymous vector which is extended with $k-\bar{k}$ zeros.
}: 
\begin{align}
  R(a)  \leq 
  \left\| \bm{g}_a^{\ell} \right\|
  \times \left\|  \bar{\btheta} - \optimal{\btheta}\right\|
  + 
  \frac{1}{2} \lambda \left( \bm{H}_{a}^{\ell} \right) 
  \times 
  \left\| \bar{\btheta} - \optimal{\btheta}\right\|^2 
  + 
  \cO\left( \left\|\bar{\btheta} - \optimal{\btheta} \right\|^3 \right),
  \label{eq:thm1}
\end{align}
where 
$\bm{g}_a^\ell = \nabla_{\btheta} J( \optimal{\btheta}; D_{a})$ is the vector of gradients associated with the loss function $\ell$ evaluated at $\optimal{\btheta}$ and computed using group data $D_a$,  
$\bm{H}_{a}^{\ell} = 
\nabla^2_{\btheta} J(\optimal{\btheta}; D_{a})$ is the Hessian matrix of the loss function $\ell$, at the optimal parameters vector $\optimal{\btheta}$, computed using the group data $D_a$ (henceforth simply referred to as \emph{group hessian}), and  
$\lambda(\Sigma)$ is the maximum eigenvalue of a matrix $\Sigma$.
\end{theorem}

The bound above follows from a second order Taylor expansion of the loss function, Cauchy-Schwarz inequality, and properties of the Rayleigh quotient.

Notice that, in addition to the difference in the original and pruned parameters vectors, two key terms appear in Equation \eqref{eq:thm1}: 
{\bf (1)} The norms of the gradients $\bm{g}_a^\ell$ and {\bf (2)} the maximum eigenvalue of the Hessian matrix $\bm{H}_a^\ell$ for a group $a$. Informally, the former is associated with the groups' local optimality while the latter relates to the ability of the model to separate the groups data. 
As we will show next these components represent the main sources of unfairness due to model pruning.

The following is an important corollary of Theorem \ref{thm:taylor}. It shows that the larger the pruning, the larger will be the excessive loss for a given group.

\begin{corollary}
\label{cor:1}
Let $\bar{k}$ and $\bar{k}'$ be the size of parameter vectors $\bar{\btheta}$ and $\bar{\btheta}'$, respectively, resulting from pruning model $f_{\optimal{\btheta}}$, where $\bar{k} < \bar{k}'$ (i.e., the former model prunes more weight than the latter one). Then, for any group $a \in \cA$, 
\begin{equation}
  \tilde{R}(a, \bar{\btheta}) \geq \tilde{R}(a, \bar{\btheta}'),
\end{equation}
where $\tilde{R}(a, \bm{\omega})$ is the excessive loss upper bound computed using pruned model parameters $\bm{\omega}$ (Eq.~\eqref{eq:thm1}).
\end{corollary}

The corollary above indicates that the excess risk for a group increases as the pruning regime increase. Building on this result, the paper illustrates next why unfairness can become more significant as the pruning regime increases.

The next sections analyze the effect of gradient norms and the Hessian to unfairness in the pruned models. The theoretical claims are supported and complemented by analytical results. These results use the UTKFace dataset \cite{zhifei2017cvpr}
for a vision task whose goal is to classify ethnicity. The experiments use a ResNet-18 architecture and the pruning counterparts remove the $P\%$  parameters with the smallest absolute values for various $P$. All reported metrics are normalized and an average of 10 repetitions. 
While the theoretical analysis focuses on the notion of disparate impacts under the lens of excessive loss, the empirical results report differences in accuracy of the resulting models. The empirical results thus reflect the setting commonly adopted when measuring accuracy parity \cite{zhao2019inherent}.

We report a glimpse of the empirical results, with the purpose of supporting the theoretical claims, and extended experiments, as well as additional descriptions of the datasets and settings, are reported in Appendix \ref{sec:additional_results}.

\section{Why disparity in groups' gradients causes unfairness?}
\label{sec:grad_analysis}

This section analyzes the effect of gradients norms on the unfairness 
observed in the pruned models.
In more detail, it shows that unbalanced datasets result in a model with 
large differences in gradient norms between groups (Proposition 
\ref{thm:grad_imbalance}), it connects gradients norms for a group with 
the resulting model errors in such a group (Proposition 
\ref{thm:acc_vs_grad}), and connects these concepts with the excessive 
loss (Theorem \ref{thm:taylor}) 
to show that unfairness in model pruning is largely controlled by the 
difference in gradient norms among groups.

\begin{wrapfigure}[9]{r}{145pt} 
    \centering
    \vspace{-20pt}
    \includegraphics[width=\linewidth]{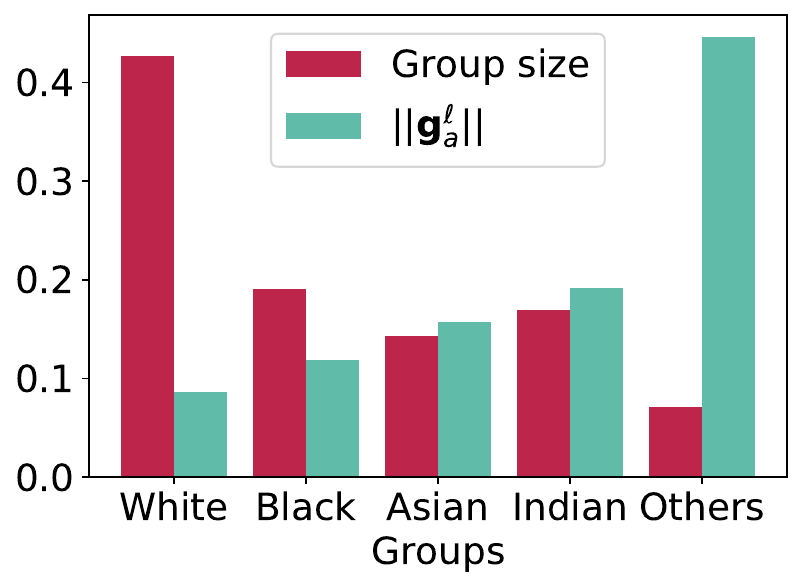}
    \vspace{-18pt}
    \caption{\small Group size vs.~gradient norms.}
    \label{fig:size_vs_grad}
\end{wrapfigure}
\paragraph{Gradient norms and group sizes.}
The section first shows that imbalanced datasets lead a model to have 
imbalanced gradient norms across groups. The following result assumes 
that the training converges to a local minima. 

\begin{proposition}
\label{thm:grad_imbalance} 
Consider two groups $a$ and $b$ in $\cA$ with $|D_a| \geq |D_b|$. Then 
\(
\left\| \bm{g}_a^\ell \right\|  \leq  \left\| \bm{g}_b^\ell \right\|.
\)
\end{proposition}
That is, groups with more data samples will result in smaller gradients 
norms than groups with fewer data samples and vice-versa. Figure \ref{fig:size_vs_grad} illustrates
Proposition \ref{thm:grad_imbalance}. The plot shows the relation between
groups sizes $|D_a|$ and their associated gradient 
norms $\|\bm{g}_a^\ell \|$ on the UTK dataset and settings described above. Notice the strong trend between decreasing group sizes and increasing gradient norms for such groups.
These theoretical considerations can be used to explain why underrepresented groups are often subject to larger performance impacts after network pruning \cite{Hooker2020CharacterisingBI}. These groups tend to exhibit large gradient norms at convergence, relative to other groups, thus, by  Theorem \ref{thm:taylor}, they are also subject to larger excessive losses due to pruning.

\paragraph{Gradient norms and accuracy.} 
Next, the section shows a strong connection between the gradient norms of a group and its associated accuracy. The following assumes the models adopt a cross entropy loss (or mean squared error for regression tasks, as shown Appendix \ref{sec:missing_proofs}). 

\begin{proposition}
\label{thm:acc_vs_grad}
For a given group $a \in \cA$, gradient norms can be upper bounded as:
\[
    \|\bm{g}_a^\ell \| \in
    \cO\left( 
    \sum_{(\bm{x}, y) \in D_a} 
    \underbrace{\|f_{\opttheta}(\bm{x}) - y \|}_{\textit{Error}}
    \times
    \left\| \nabla_{\btheta} f_{\opttheta}(\bm{x}) \right\|
    \right).
\]
\end{proposition}
The above relates gradient norms with an error measure of the classifier 
to a target label multiplied by the gradient of the predictions. 
For example, in a classification task with cross entropy loss, 
$\ell(f_{\btheta}(\bm{x}), y) = - \sum_{z \in \cY} f^z_{\btheta}
(\bm{x}) \bm{y}^z$, where $f_{\btheta}^z(\bm{x})$ represents the $z$-th element of the output associated with the soft-max 
layer of model $f_\btheta$, and $\bm{y}$ is a one-hot encoding of the true label $y$, with $\bm{y}^z$ representing its $z$-th element, then,
\begin{align*}
\hspace{-50pt}
    \| \bm{g}_a\|  
    &= \left\| \nabla_{\btheta} J(\btheta; D_a,) \right\| 
    = \left\| \nicefrac{1}{|D_a|} \sum_{(\bm{x},y) \in D_{a}} 
    \nabla_{f} \ell(f_{\btheta}(\bm{x}), y) \times 
    \nabla_{\btheta} f_{\btheta}(\bm{x}) \right\|\\
\hspace{-50pt}    
    &= \left\| \nicefrac{1}{|D_a|} \sum_{(\bm{x}, y) \in D_a} 
    (f_{\btheta}(\bm{x}) - \bm{y}) \times  
    \nabla_{\btheta} f_{\btheta}(\bm{x}) \right\|\\
\hspace{-50pt}    
    &\leq \nicefrac{1}{|D_a|}\sum_{(\bm{x}, y) \in D_a}  
    \left\| f_{\btheta}(\bm{x}) - \bm{y}  \right\| \times 
    \left\| \nabla_{\btheta} f_{\btheta}(\bm{x}) \right\|.
\end{align*}
\begin{wrapfigure}[4]{r}{145pt} 
    \vspace{-80pt}
    \centering
    \includegraphics[width=\linewidth]{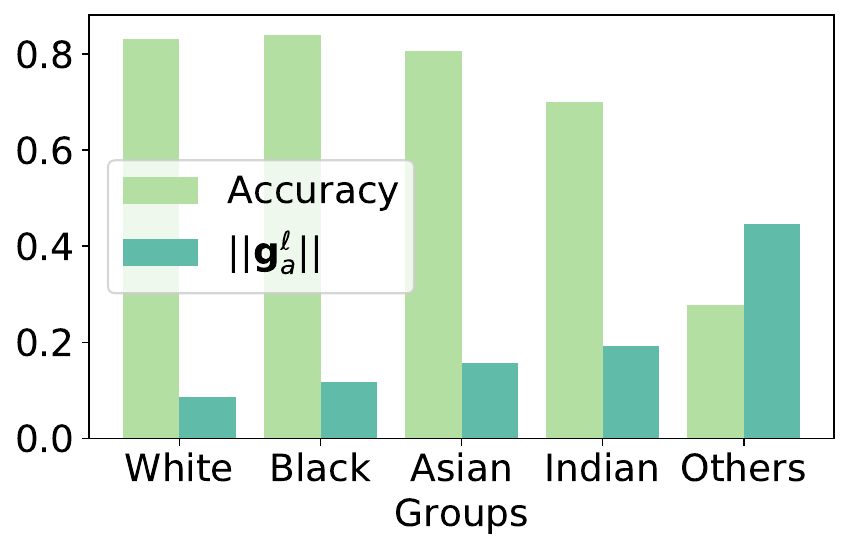}
    \vspace{-16pt}
    \caption{\small Accuracy vs. gradient norms.}
    \label{fig:acc_vs_norms}
\end{wrapfigure}
A similar observation holds for mean square error loss, as illustrated 
in Appendix \ref{sec:missing_proofs}.
The observation above sheds light on the correlation between the prediction error of a group and its model gradients. 
This relation is emphasized in Figure~\ref{fig:acc_vs_norms}, which 
illustrates that the gradient norm for a given group increases as its prediction accuracy decreases.\!

Proposition \ref{thm:acc_vs_grad} allows us to link the gradient norms with the group accuracy of the resulting model, which, together with the result above will be useful to reason about the impact of gradient norms on the disparities in the group excessive losses.

\paragraph{The role of gradient norms in pruning.}
Having highlighted the connection between  gradients norms of a group 
with the accuracy of the pruned model on such a group, this section provides theoretical intuitions on the role of gradient norms in the disparate group losses during pruning. 

From Theorem \ref{thm:taylor}, notice that the excessive loss is controlled 
by term $\|\bm{g}_a^\ell \| \times \| \bartheta - \opttheta \|$. As already 
noted in Corollary \ref{cor:1}, the term $\| \bartheta - \opttheta \|$
regulates the impact of pruning on the excessive loss, as the difference between the pruned and non-pruned parameters vectors directly depends on the pruning rate. For a fixed pruning rate, however, notice that 
groups with different gradient norms will have a disparate effect on 
the resulting term. In particular, groups with very small gradients norms 
(those generally associated with highly accurate predictions) will be 
less sensitive to the effects of the pruning rate. Conversely, groups with large gradient norms will be affected by the pruning rate to a greater extent, with larger pruning rates, \emph{typically} reflecting in larger excessive losses.

\begin{figure}[!t]
\centering 
\begin{subfigure}[t]{0.28\textwidth}
\centering
\includegraphics[width=\textwidth]{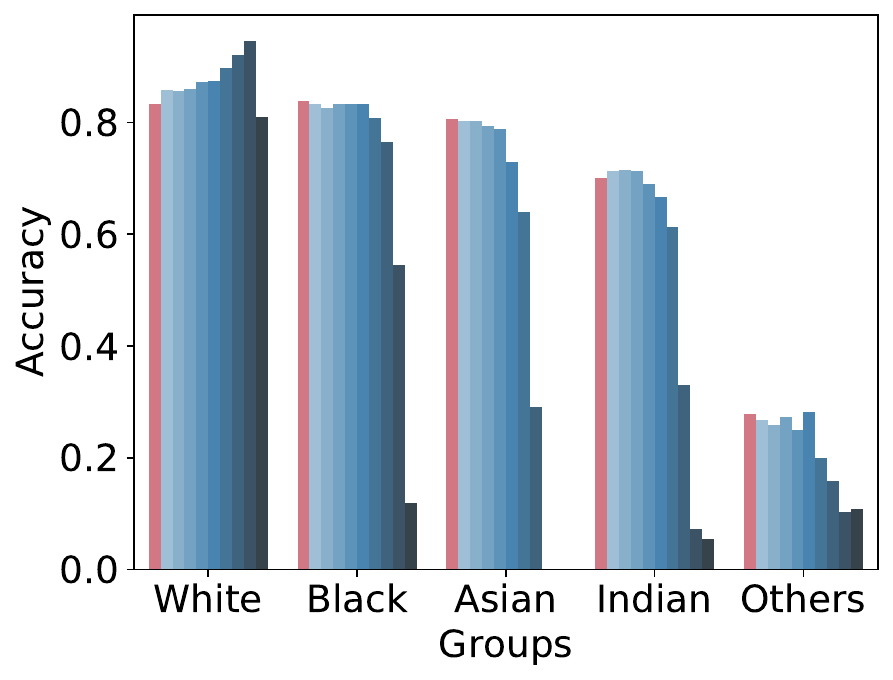}
\caption{Accuracy}
\label{fig:global_pruning_accuracy_avg}
\end{subfigure}
\hfill
\begin{subfigure}[t]{0.275\textwidth}
\centering
\includegraphics[width=\textwidth]{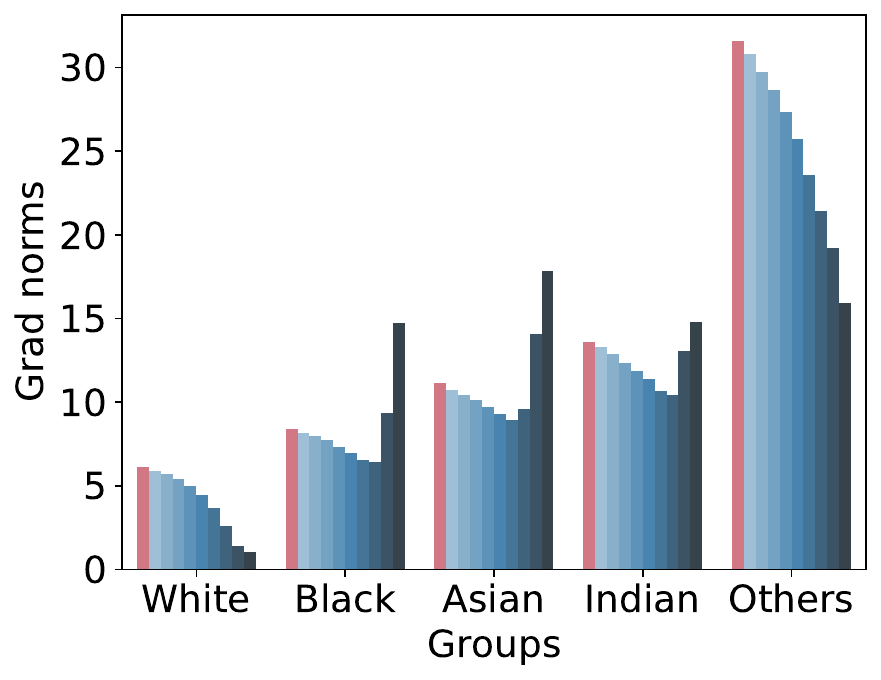}
\caption{Gradient Norm: $\|\bm{g}_a^\ell\|$}
\label{fig:nonnormal_GN}
\end{subfigure}
\hfill
\begin{subfigure}[t]{0.425\textwidth}
\centering
\includegraphics[width=\textwidth]
 {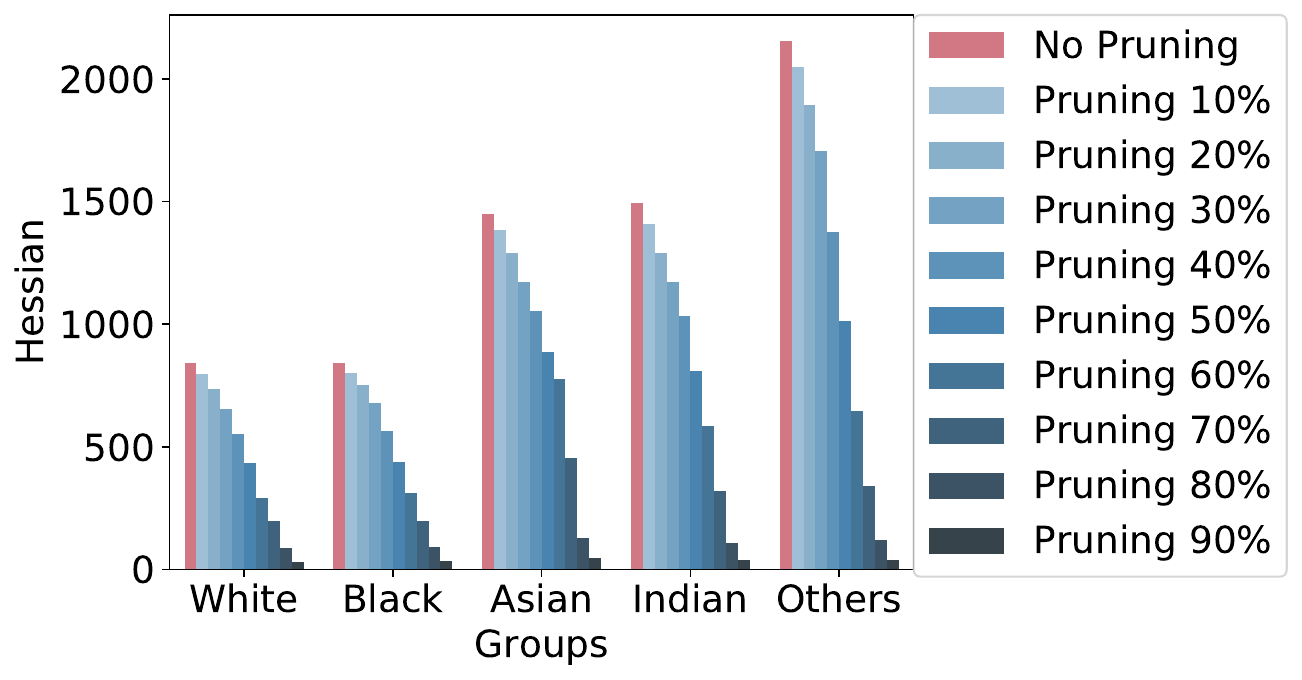}
\caption{Group Hessian: $\lambda(\bm{H}_a^\ell)$}
\label{fig:nonnormal_hessian}
\end{subfigure}
\caption{Accuracy, gradient norm, and group Hessian max eigenvalues of each ethnicity group, before and after increasing pruning ratios for UTK-Face dataset. The percentage of data samples across groups \textit{White, 
 Black, Asian, Indian, and Others} is $\sim 0.42, 0.19, 0.15, 0.15, 0.07$,
 respectively.}
\label{Fig:main_fig}
\end{figure}
These observations of the factors of disparity, accuracy, and group size,
can also be appreciated empirically in 
Figures~\ref{fig:global_pruning_accuracy_avg} and \ref{fig:nonnormal_GN}. 
The plots report accuracy (a) and gradient norms (b) on the UTKFace datasets for a variety of pruning rates. Consider group {\sl White} (containing 42\% of the total samples) and {\sl Others} (containing 7\% of the total samples). The unpruned model 
has high accuracy on the former group and small gradient norms. The accuracy of this group is insensitive to various pruning rates and even increases 
at large pruning regimes. 
In contrast, group {\sl Others} has much lower accuracy and larger 
gradient norms in the unpruned model. As the pruning rate increase, their accuracies drastically drop. As a result, in high pruning regimes, this minority group exhibits poor accuracy and very high gradient norms.

Notice that the empirical results apply to much more complex settings
than those which can be analyzed formally, thus they complement the theoretical observations.

\section{Why disparity in groups' Hessians causes unfairness?}
\label{sec:hessian_analysis}

Having examined the properties of the groups gradients and their 
relation to unfairness in pruning, this section turns on analyzing 
how the 
Hessian associated with the loss function for a group is linked to the unfairness observed during pruning. In more detail, it connects the groups' Hessian to the distance to the decision boundary for the samples in that group and their resulting model errors (Theorem \ref{thm:hessian_norm_bound}), it illustrates a strong positive correlation between groups' Hessian and gradient norms, and links these concepts with the excessive loss (Theorem \ref{thm:taylor}) to show that unfairness in model pruning is controlled by the difference in maximum eigenvalues of the Hessians among groups.

\paragraph{Group Hessians and accuracy.}
The section first shows that groups presenting large Hessian values may suffer larger disparate impacts due to pruning, when compared with groups that have smaller Hessians. It does so by connecting the maximum eigenvalues of the groups Hessians with their distance to decision boundary and the group accuracy. The following result sheds light on these observations. It restricts its attention to models trained under binary cross entropy losses, for clarity of explanation, although an extension to a multi-class case is directly attainable. 
\begin{theorem}
\label{thm:hessian_norm_bound} 
Let $f_\theta$ be a binary classifier trained using a binary cross entropy loss. For any group $a \in \cA$, the maximum eigenvalue of the group Hessian $\lambda(\bm{H}_a^{\ell})$ can be upper bounded by:
\begin{equation}
    \lambda(\bm{H}_a^{\ell}) \leq \frac{1}{|D_a|}
    \sum_{(\bm{x}, y)\in D_a}
    \underbrace{
        \left( {f}_{\opttheta}(\bm{x}) \right) 
        \left( 1 - {f}_{\opttheta}(\bm{x}) \right)}_{\textit{Closeness to decision boundary}} 
        \times 
        \left\| \nabla_{\btheta} f_{\opttheta}(\bm {x}) \right\|^2 
     + 
         \underbrace{\left| f_{\opttheta}(\bm{x}) - y \right|}_{\textit{Error}} 
         \times 
        \lambda\left( \nabla^2_{\btheta} f_{\opttheta}(\bm{x}) \right). 
    \label{eq:hessian_norm_bound}
\end{equation}

\end{theorem}
The proof relies on derivations of the Hessian associated with model loss function and Weyl inequality. 
In other words, Theorem \ref{thm:hessian_norm_bound} highlights a direct connection between the maximum eigenvalue of the group Hessian and {\bf (1)} the closeness to the decision boundary of the group samples, and {\bf (2)} the accuracy of the group. The distance to the decision boundary is derived from \cite{cohen2019certified}. Intuitively this term is maximized when the classifier is highly uncertain about the prediction: ${f}_{\opttheta}(\bm{x}) \to 0.5$, and minimized when
it is highly certain ${f}_{\opttheta}(\bm{x}) \to 0$ or $1$, as showed in the following proposition.

\begin{proposition}
\label{prop:dist_bndry}
 Consider a binary  classifier $f_{\btheta}(\bm{x}) $. For a given sample $\bm{x} \in D $, the term ${f}_{\opttheta}(\bm{x}) (1 - {f}_{\opttheta}(\bm{x}) )$ is maximized when ${f}_{\opttheta}(\bm{x})  = 0.5$ and minimized when ${f}_{\opttheta}(\bm{x}) \in \{0,1\} $.
\end{proposition}

\begin{wrapfigure}[10]{r}{143pt} 
    \centering
    \vspace{-12pt}
    \includegraphics[width=\linewidth]{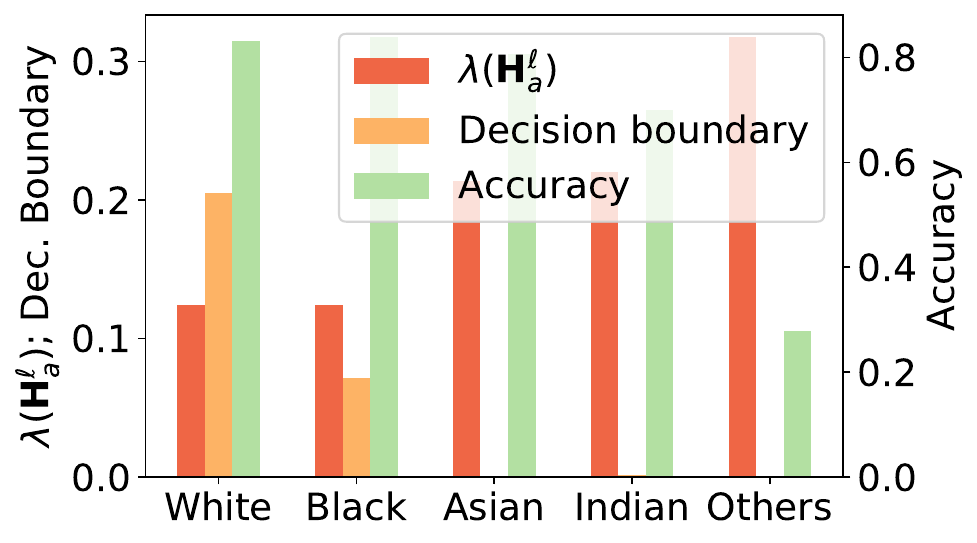}
    \vspace{-18pt}
    \caption{\small Group Hessians, distance to decision boundary, and~accuracy.}
    \label{fig:H_vs_acc}
\end{wrapfigure}
Observe that a group consisting of samples that are far from the decision boundary will have smaller Hessians and, thus, be less subject to a drop in accuracy due to model pruning. These results can be appreciated in Figure \ref{fig:H_vs_acc}. Notice the inverse relationship between maximum eigenvalues of the groups' Hessians and their average distance to the decision boundary. The same relation also holds for accuracy: the higher the Hessians maximum eigenvalues, the smaller the accuracy.
This is intuitive as samples which are close to the decision boundary
will be more prone to errors due to small changes in the model due 
to pruning, when compared with samples lying far from the decision boundary. 

\paragraph{Correlation between group Hessians and gradient norms.}
This section observes a positive correlation between maximum eigenvalues of the Hessian of a group and their gradient norms. This relation can be appreciated in Figure \ref{fig:H_vs_grad}. 
While mainly empirical, this observation is important as it illustrates that both the Hessian $\lambda(\bm{H}_a^\ell)$ and the gradient $\| \bm{g}_a^\ell \|$ terms appearing in the upper bound of the excessive loss $R(a)$ reported in Theorem \ref{thm:taylor} are in agreement. This relation was observed in all our experiments and settings.
Such observation allows us to infer that it is the combined 
effect of gradient norms and group Hessians that is responsible for 
the excessive loss of a group and, in turn, for the exacerbation of 
unfairness in the pruned models.

\begin{wrapfigure}[11]{r}{130pt} 
    \centering
    \vspace{-14pt}
    \includegraphics[width=\linewidth]{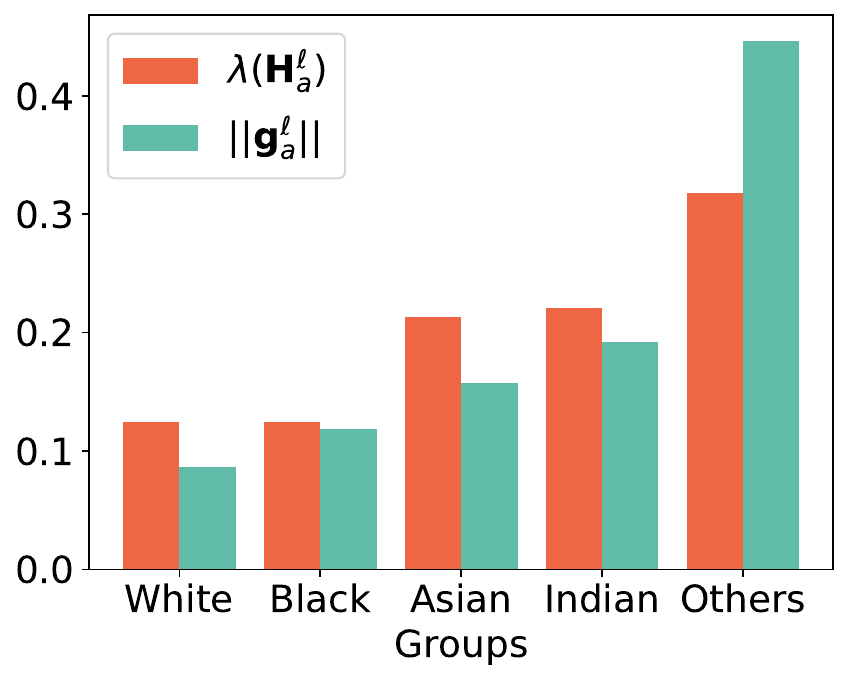}
    \vspace{-18pt}
    \caption{\small Group Hessians and~gradient norms.}
    \label{fig:H_vs_grad}
\end{wrapfigure}
\paragraph{The role of the group Hessian in pruning.}
 Having highlighted the connection between Hessian for a group with the resulting accuracy of the model on such a group, this section provides theoretical intuitions on the role of the Hessians in the disparate group losses during pruning. 

In Theorem \ref{thm:taylor}, notice that the excessive loss is controlled by term $\|\bm{H}_a^\ell \| \times \| \bartheta - \opttheta  \|^2$. 
As also noted in the previous section, the term $\| \bartheta - \opttheta  \|$
regulates the impact of pruning on the excessive loss as the difference between the pruned and non-pruned parameter vectors directly depends on the pruning rate. Similar to the observation for gradient norms, with a fixed pruning rate, groups with different Hessians will have a disparate effect on the resulting term. In particular, groups with small Hessians eigenvalues (those generally distant from the decision boundary and highly accurate) will be less sensitive to the effects of the pruning rate. Conversely, groups with large Hessians eigenvalues will be affected by the pruning rate to a greater extent, \emph{typically} resulting in larger excessive losses. 
These observations can further be appreciated empirically in Figures~\ref{fig:global_pruning_accuracy_avg} (for accuracy) and \ref{fig:nonnormal_hessian} (for maximum group Hessian eigenvalues) on the UTKFace datasets for a variety of pruning rates.

\section{Mitigation solution and evaluation}
\label{sec:mitigation}

The previous sections highlighted the presence of two key factors playing a role in the observed model accuracy disparities due to pruning: the difference in gradient norms, and the difference in Hessians losses across groups. This section first shows how to leverage these findings to provide a simple, yet effective solution to reduce the disparate impacts of pruning. Then, the section illustrates the benefits of this mitigating solution on a variety of tasks, datasets, and network architectures. 

\subsection{Mitigation solution} 
To achieve fairness, the aforementioned findings suggest to equalize the disparity associated with gradient norms $\| \bm{g}_a^\ell \|$ and Hessians
 $\lambda(\bm{H}_a^\ell)$ across different groups $a \in \cA$. For this goal, we adopt a constrained empirical risk minimization approach: 
\begin{align}
\label{eq:LD_slow}
    \minimize{\bm{\theta}}\;\; J(\bm{\theta}; D) 
    \quad
    \text{such that:}\;\; \| \bm{g}^\ell_a\| = \|\bm{g}^\ell\|, \;\;
    \lambda(\bm{H}_a^\ell) = \lambda(\bm{H}^\ell) \;\; \forall a \in \cA,
\end{align} 
where $\bm{g}^\ell = \nabla_\btheta J(\btheta; D)$ and $\bm{H}^\ell = \nabla^2_\btheta J(\btheta; D)$ refer to the gradients and Hessian associated with loss function $\ell$, respectively, and are computed using the whole dataset $D$. The approach \eqref{eq:LD_slow} is a common strategy adopted in fair learning tasks, and
the paper uses the Lagrangian Dual method of \citet{fioretto2020lagrangian} which exploits Lagrangian duality to extend the loss function with trainable and weighted regularization terms that encapsulate constraints violations (see Appendix \ref{sec:additional_results} for additional details). 

A shortcoming of this approach is, however, that requires computing the gradient norms and Hessian matrices of the group losses in each and every training iteration, rendering the process computationally unviable, especially for deep, overparametrized networks.

To overcome this computational burden, we will use two observations made earlier in the paper. First, recall the strong relation between gradient norms for a group and their associated losses. This aspect was noted in Proposition \ref{thm:acc_vs_grad}. That is, when the losses across the groups are similar, the gradient norms across such groups will also tend to be similar. Next, Theorem \ref{thm:hessian_norm_bound} noted a positive correlation between model errors (and thus loss values) for a group and its associated Hessian eigenvalues. Thus, when the losses across the groups are similar, the group Hessians will also tend to be similar. This intuition is also complemented by the strong correlation between group Hessians and gradient norms reported in Section \ref{sec:hessian_analysis}.
Based on the above observations, we propose a simpler version of the constrained minimizer \eqref{eq:LD_slow} defined as
\begin{align}
\label{eq:LD_fast}
    \minimize{\bm{\theta}}\;\; J(\bm{\theta}; D) 
    \quad
    \text{such that:}\;\; J(\bm{\theta}; D_a) = J(\bm{\theta}; D)  
    \;\; \forall a \in \cA,
\end{align}
that substitutes the gradient norms and max eigenvalues of group Hessians equality constraints with proxy terms capturing the group $J(\btheta; D_a)$ and population $J(\btheta; D)$ losses.

\begin{wrapfigure}[17]{r}{0.5\textwidth}
\vspace{-16pt}
\centering
\includegraphics[width=\linewidth]{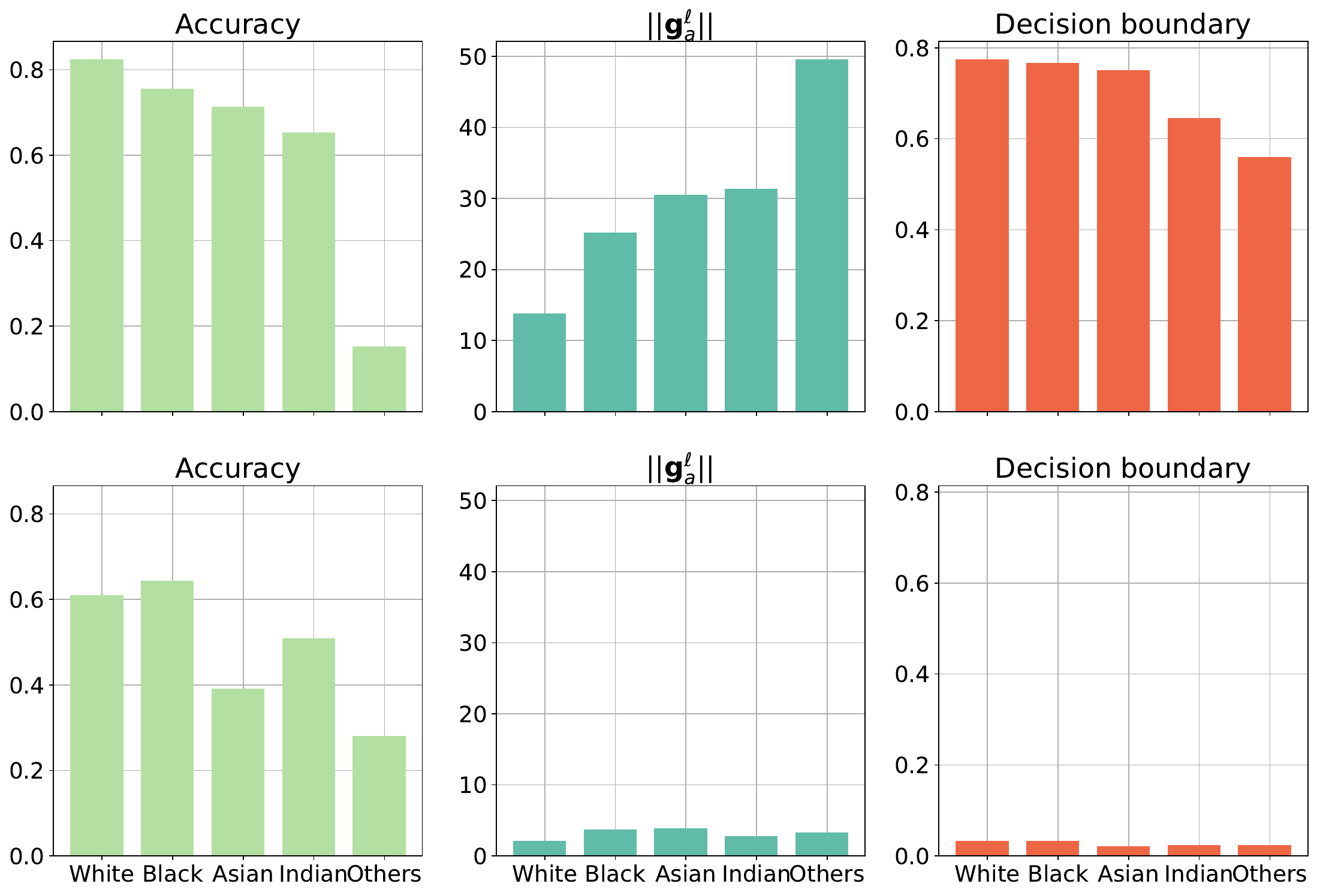}
\vspace{-12pt}
\caption{Effects of fairness constraints in balancing not only group accuracy (left) but also gradient norms (middle) and group average distance to the decision boundary (right).}
\label{fig:example_mitigation}
\end{wrapfigure}
The impact of such proxy terms in the fairness-constrained problem above can be appreciated, empirically, in Figure \ref{fig:example_mitigation}. The plots, that use the UTK-Face dataset, with Ethnicity as protected group, show an original unfair model (top) and a fair counterpart obtained through problem \eqref{eq:LD_fast} (bottom). 
Both top and bottom sub-figures use an unpruned model. The top sub-figure shows the performance of an original unpruned model trained by minimizing the empirical risk function while the bottom one shows the effect of solving Problem \eqref{eq:LD_fast}, i.e., it constrains the empirical risk function with the various group loss terms. 
Notice how enforcing balance in the group losses also helps reducing and balancing the gradient norms and group's  average distance to the decision boundary. As a consequence, the resulting model fairness is dramatically enhanced (bottom-left subplot).

\subsection{Assessment of the mitigation solution}\label{sec:AssmtMitigation}

\paragraph{Datasets, models, and settings.}
This section analyzes the results obtained using the proposed mitigation solution with ResNet50 and VGG19 on the UTKFace dataset \cite{zhifei2017cvpr}, CIFAR-10~\cite{cifar10}, and SVHN~\cite{Netzer2011ReadingDI} for various protected attributes. The experiments compare the following four models:

$\bullet$ {\sl No Mitigation}: it refers to the standard pruning approach which uses no fairness mitigation strategy.\\ 
$\bullet$ {\sl Fair Bf Pruning}: it applies the fairness mitigation process (Problem \eqref{eq:LD_fast}) exclusively to the original large network, thus \emph{before} pruning.\\
$\bullet$ {\sl Fair Aft Pruning}: it applies the mitigation exclusively to the pruned network, thus \emph{after} pruning. \\
$\bullet$ {\sl Fair Both}: it applies the mitigation both to the original large network and to the pruned network.

The experiments report the overall accuracy of resulting models as well as their fairness violations, defined here as the difference between the maximal and minimal group accuracy. The reported metrics are the average of $10$ repetitions. Additional details on datasets, architectures, and hyper-parameters adopted, as well as additional and extended results are reported in Appendix \ref{sec:additional_results}.

\begin{figure*}[t]
 \centering
\includegraphics[width=0.45\linewidth]
 {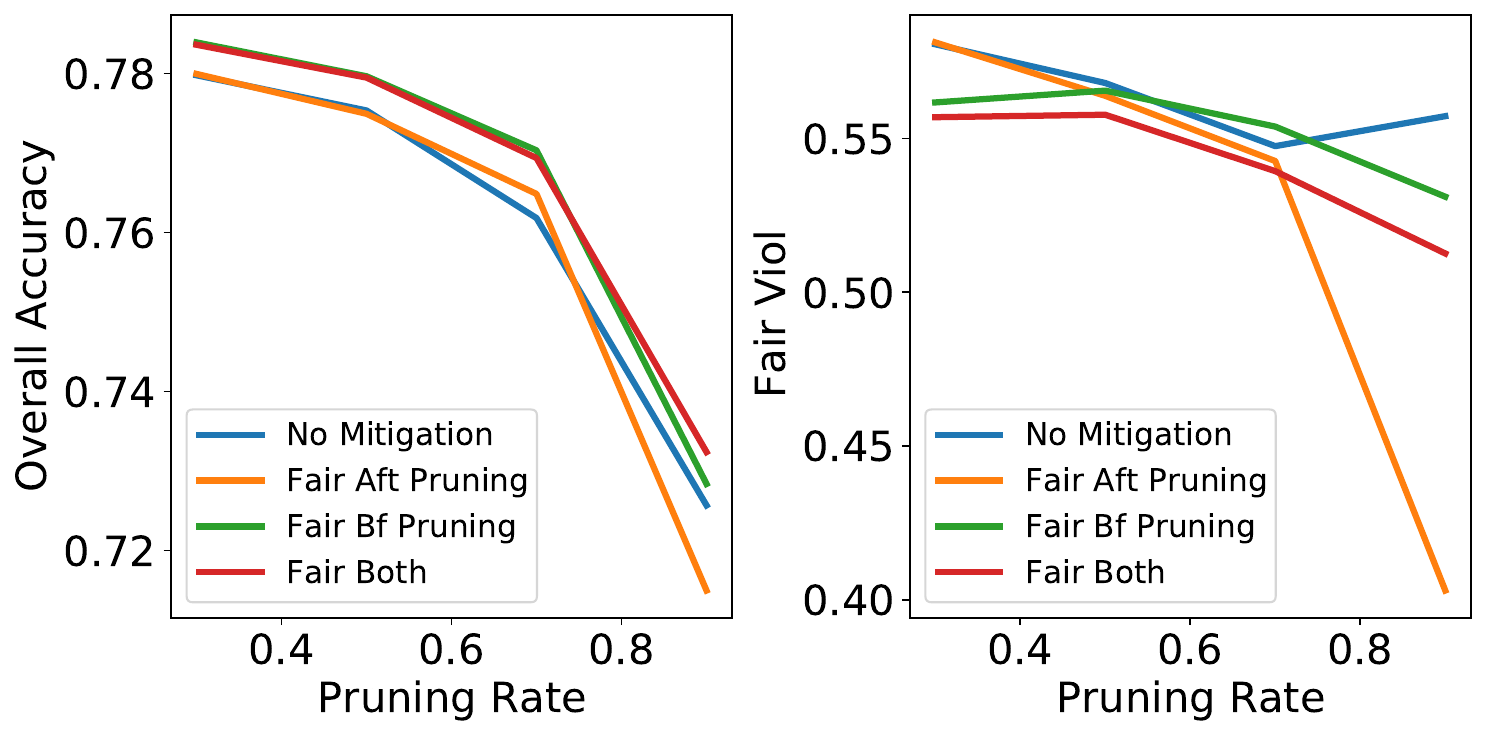}
 \hspace{20pt}
\includegraphics[width=0.45\linewidth]
 {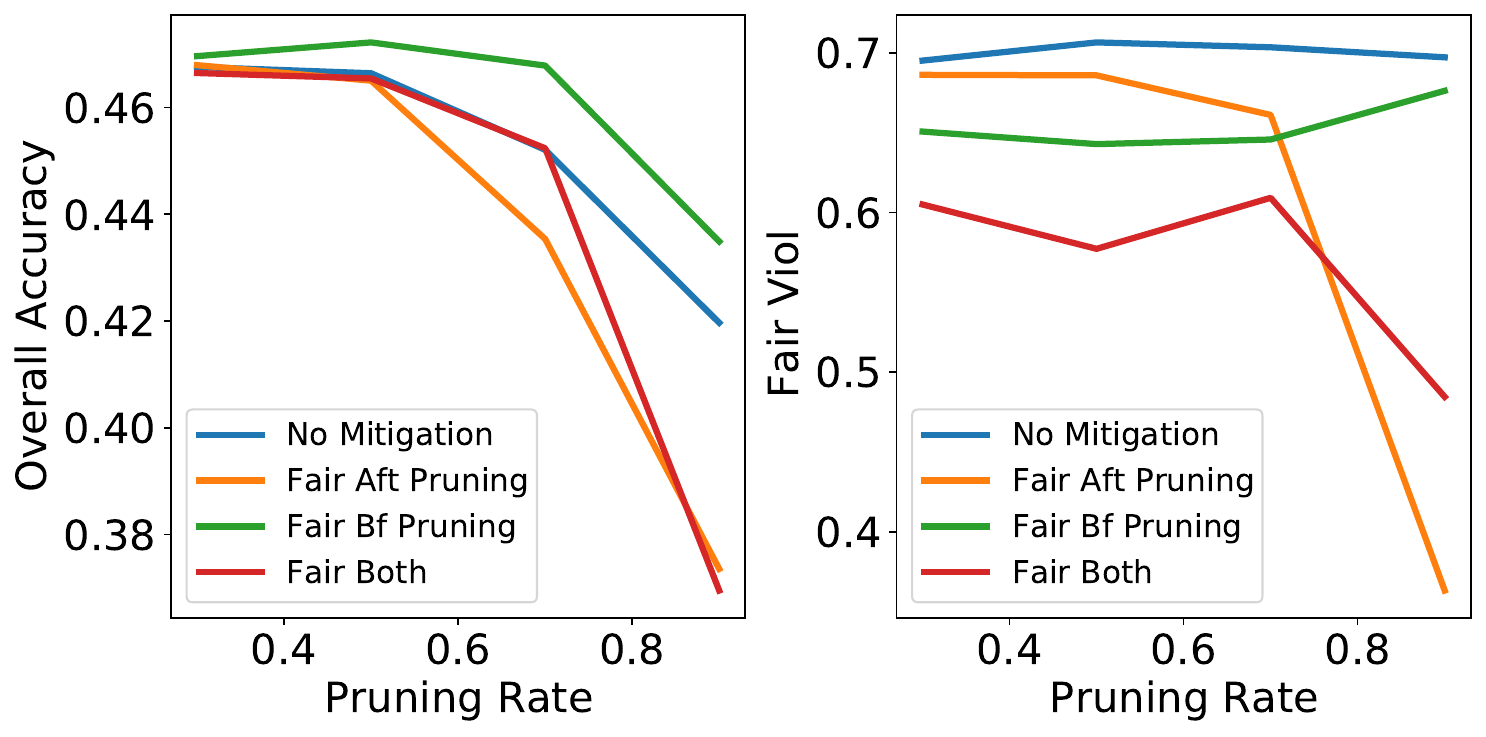}
\caption{Accuracy and Fairness violations attained by all models on ResNet50, UTK-Face dataset with {\sl ethnicity} (5 classes) as group attribute (and labels) [left]  and {\sl age} (9 classes) [right].}
 \label{fig:mitigation_pruning-1}
\end{figure*}

\begin{figure*}[t]
 \centering
\includegraphics[width=0.45\linewidth]
 {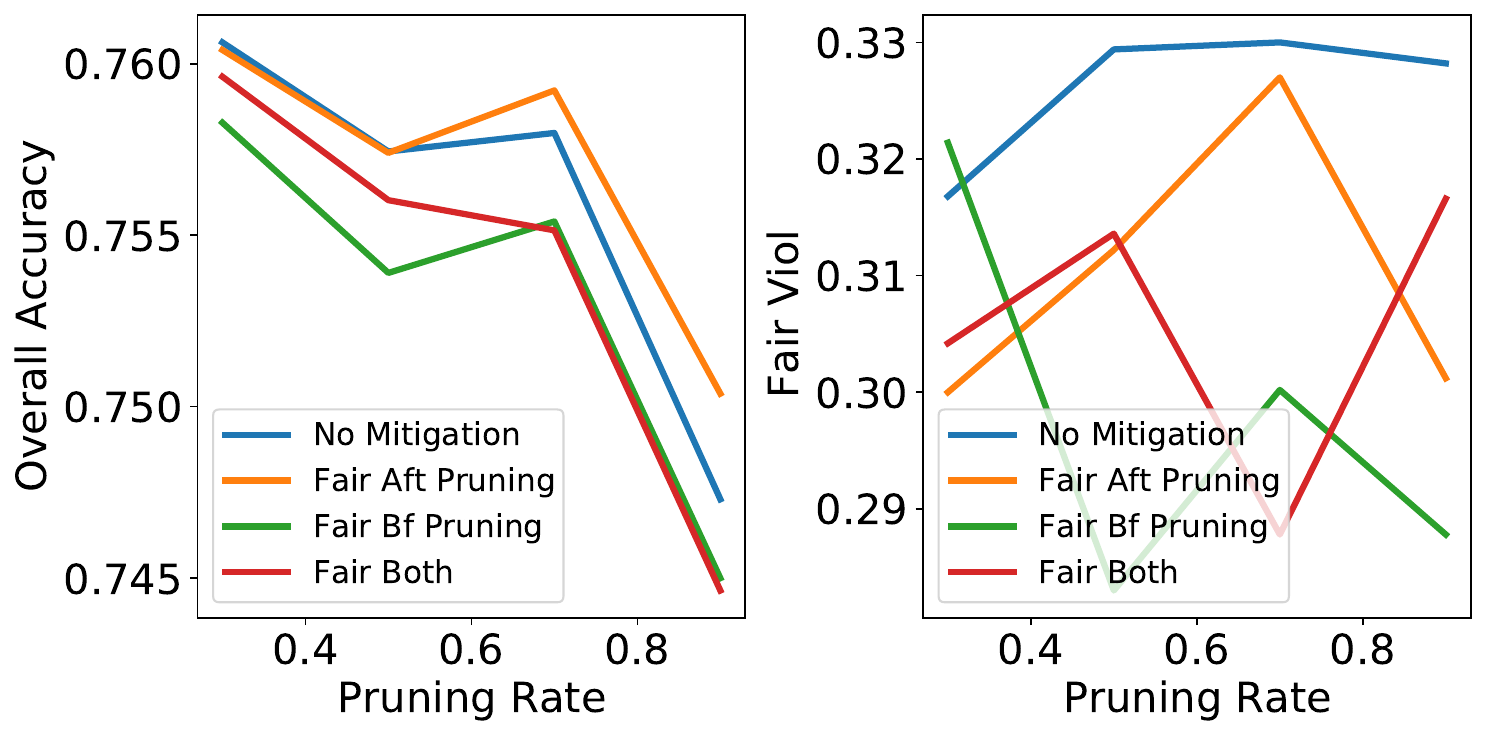}
 \hspace{20pt}
\includegraphics[width=0.45\linewidth]
 {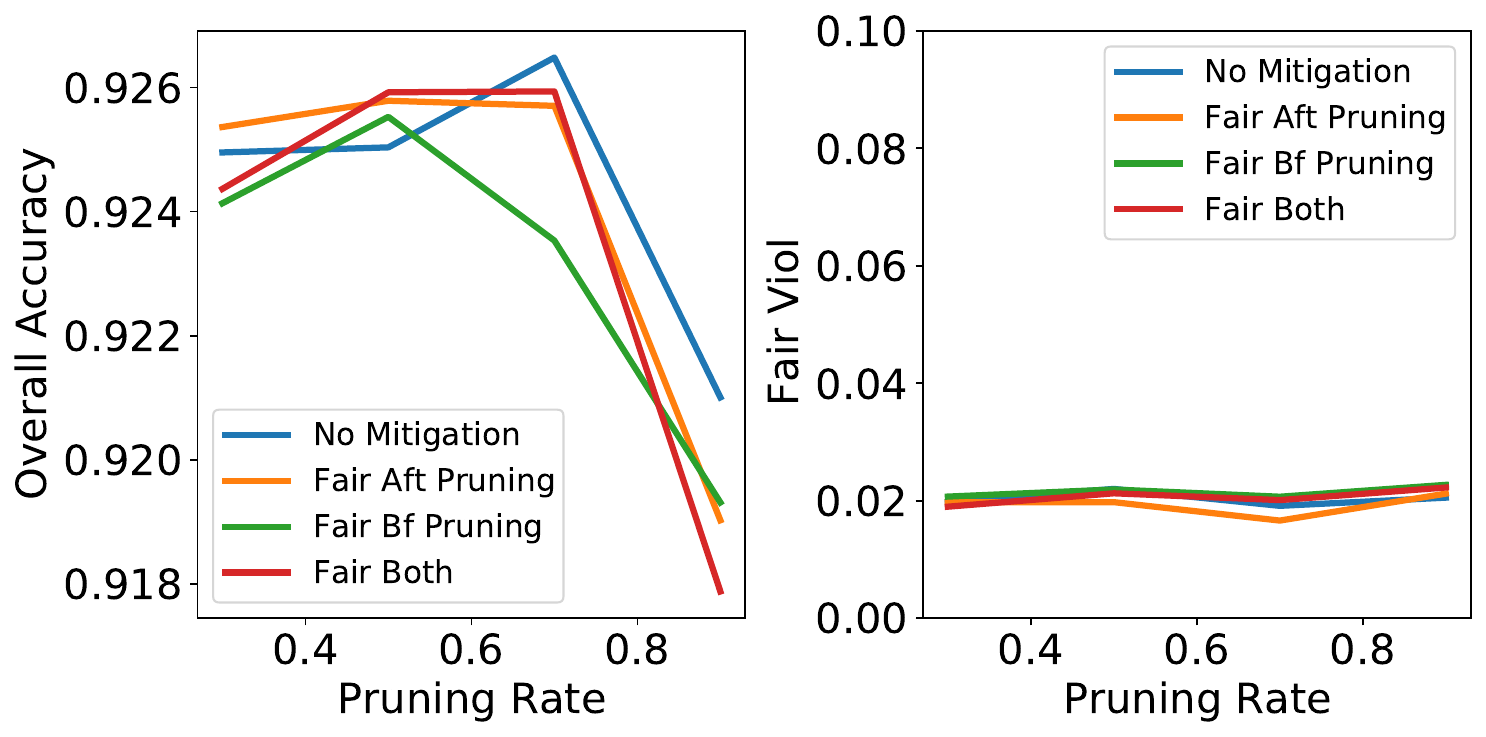}
\caption{Accuracy and Fairness violations attained by all models on VGG-19, CIFAR-10 dataset (left) and SVHN (right) with 10 class labels also used as group attribute.}
 \label{fig:mitigation_pruning-2}
\end{figure*}

\paragraph{Effects on accuracy.}
The section first focuses on analyzing the effects of accuracy drop due to applying the proposed mitigation solution for fair pruning. Figure \ref{fig:mitigation_pruning-1} compares the four models on the UTK-Face dataset using a ResNet50 architecture. The left subplots use {\sl ethnicity} as protected group and class label, with $|\cY| = 5$, while the right subplots use {\sl age} as protected group and class label, with $|\cY| = 9$. Notice that, as expected, all compared models present some accuracy deterioration as the pruning rate increases. However, notably, the deterioration of the models that apply the fair mitigation steps are comparable to (or even improved) those of the "{\sl No mitigation}" model, which applies standard pruning. 

A similar trend can be seen in Figure \ref{fig:mitigation_pruning-2} that reports results on CIFAR (left) and SVHN (right). Both use the ten class labels as protected attributes. These results clearly illustrate the ability of the mitigating solution to preserve highly accurate models. 

A comparison of the ``full'' (Equation \ref{eq:LD_slow}) and ``relaxed'' (Equation \ref{eq:LD_fast}) versions of the proposed mitigation solutions is 
provided in Table~\ref{tab:full_relaxed_results}. We note that while the "full" version leads to fairer results, the reduction in the various groups accuracy is often insubstantial. We also note that the running time of the "full" version is largely (over an order magnitude) longer than the relaxed counterpart. This is due to the calculation of gradient norms and the Hessian terms associated with each group.

\begin{table}[t]
\centering
\resizebox{0.85\columnwidth}{!}{%
\begin{tabular}{@{}rr@{\hspace{10pt}}ll@{}}
\toprule
{\bf Dataset} & {\bf version} & {\bf Class-wise accuracy} & {\bf Overall accuracy} \\
  \midrule
\multirow{2}{*}{\bf UTK-age bins}   & {\bf full}    & 0.856, 0.128, 0.145, 0.319, 0.331, 0.342, 0.181, 0.334, 0.512 & 0.395 \\
                                    & {\bf relaxed} & 0.810, 0.096, 0.141, 0.284, 0.385, 0.324, 0.227, 0.257, 0.533 & 0.390 \\[0.25em]
\multirow{2}{*}{\bf UTK-gender}	    & {\bf full}    & 0.830, 0.876	& 0.857\\
                                    & {\bf relaxed} & 0.868, 0.845  & 0.852\\[0.25em]
\multirow{2}{*}{\bf SVHN}	        & {\bf full}    & 0.864, 0.911, 0.869, 0.819, 0.887, 0.784, 0.840, 0.877, 0.805, 0.856 & 0.857\\ 
                                    & {\bf relaxed} & 0.824, 0.910, 0.775, 0.726, 0.827, 0.752, 0.747, 0.789, 0.713, 0.755 & 0.795\\[0.25em]
\multirow{2}{*}{\bf MNIST}	        & {\bf full}    & 0.998, 0.996, 0.993, 0.998, 0.994, 0.991, 0.991, 0.993, 0.992, 0.985 & 0.993\\
                                    & {\bf relaxed} & 0.994, 0.988, 0.989, 0.986, 0.987, 0.979, 0.981, 0.988, 0.969, 0.994 & 0.986\\
\bottomrule
\end{tabular}%
}
\caption{Full (Equation \ref{eq:LD_slow}) vs relaxed (Equation \ref{eq:LD_fast}) versions of the proposed mitigation solutions.}
\label{tab:full_relaxed_results}
\end{table}

\paragraph{Effects on fairness.}
The section next illustrates the ability of the proposed solution to achieve fair pruned models. 
Table \ref{tab:protected_attr_results} illustrates the results for the UTKFaces dataset with ethnicity as class labels and age as 
protected attributes for a CNN with two convolutional layers and three linear layers and prune amounts: 30\%, 50\%, 70\%, and 90\%. 
Notice how Fair Aft pruning and Fair both achieve relatively lesser fairness violations compared to the No mitigation and 
the Fair bf Pruning methods in most cases.

\begin{table}[t]
\centering
\resizebox{0.6\textwidth}{!}{%
\begin{tabular}{@{}r@{\hspace{10pt}} l@{\hspace{5pt}}l@{\hspace{5pt}}l@{\hspace{5pt}}l @{\hspace{10pt}} l@{\hspace{5pt}}l@{\hspace{5pt}}l@{\hspace{5pt}}l@{}}
\toprule
{\bf Methods}          & \multicolumn{4}{c}{\bf Overall accuracy}           & \multicolumn{4}{c}{\bf Fairness violations} \\ 
                       & 30\% & 50\% & 70\% & 90 \% & 30\% & 50\% & 70\% & 90\%  \\
\midrule

{\bf No mitigation}    & 0.546 & 0.545 & 0.529 & 0.559 & 0.179 & 0.186 & 0.152 & 0.134 \\
{\bf Fair bf Pruning}  & 0.539 & 0.557 & 0.529 & 0.540 & 0.189 & 0.190 & 0.174 & 0.238 \\
{\bf Fair Aft Pruning} & 0.538 & 0.532 & 0.497 & 0.472 & 0.172 & 0.161 & 0.163 & 0.05  \\
{\bf Fair both}        & 0.525 & 0.541 & 0.508 & 0.484 & 0.170 & 0.144 & 0.156 & 0.073 \\ 
\bottomrule
\end{tabular}%
}
\caption{Accuracy and fairness violations for the UTKFaces dataset with \emph{ethnicity} as class labels and \emph{age} as protected attributes and prune amounts of 30\%, 50\%, 70\%, and 90\%.}
\label{tab:protected_attr_results}
\end{table}

Next, the second and fourth subplots presented in Figures \ref{fig:mitigation_pruning-1} and \ref{fig:mitigation_pruning-2} illustrate the fairness violations obtained by the four models analyzed on different datasets and settings. 
We make the following observations: First, all the plots exhibit a consistent trend in that the mitigation solution produces models which improve the fairness of the baseline, "{\sl No mitigation}" model. 
Observe that, as already illustrated in Figure \ref{fig:example_mitigation}, the fair models tend to equalize the gradient norms and group Hessians components (and thus the distance to the decision boundary across groups). Thus, the resulting pruned models also attain better fairness, when compared to their standard counterparts. 

Next, notice that "{\sl Fair Aft Pruning}" often achieves better fairness violations than "{\sl Fair Bf Pruning}", especially at high pruning regimes. This is because the former has the advantage to apply the mitigation solution directly to the pruned model to ensure that the resulting model has low differences in gradient norms and group Hessians. The presentation also illustrates the application of the mitigation strategies both before and after pruning ({\sl Fair Both}) which shows once again the significance of applying the mitigation solution over the pruned network. 

Finally, it is notable that "{\sl Fair Aft Pruning}" achieves good reductions in fairness violation. Indeed, pre-trained large (non-pruned) fair models may not be available and the ability to retrain these large models prior to pruning may be hindered by their size and complexity.

\section{Discussion and limitations}
\label{sec:limitations}
This section discusses three key messages found in this study. First, we notice that pruning affecting model separability and distance to the decision boundary is related to concepts also explored in robust machine learning \cite{https://doi.org/10.48550/arxiv.1412.6572,papernot2016transferability}. Not surprisingly, some recent literature in network pruning has empirically observed that pruning may have a negative impact on adversarial robustness \cite{guo2019sparse}. These observations raise questions about the connection between pruning, robustness, and fairness, which we believe is an important direction to further investigate. 

Next, although the solution proposed in Problem \eqref{eq:LD_fast} allows it to be adopted in large models, the size of modern ML models (together with the amount of hyperparameters searches) may hinder retraining such original massive models from incorporating fairness constraints. Notably, however, the proposed mitigation solution can be used as a post-processing step to be applied during the pruning operation directly. The previous section shows that the proposed method delivers desirable performance in terms of both accuracy and fairness.

Finally, we notice that the results analyzed in this paper pertain to losses that are twice differentiable. Lifting such an assumption will be an interesting and challenging future research avenue. 

\paragraph{Ethical considerations.}
The analyses and solutions reported in this paper should not be intended as an endorsement for using the developed techniques to aid facial recognition systems. We hope this work creates further awareness of the unfairness caused by pruning mechanisms in ML systems in the context of models that could be deployed in energy-efficient devices, such as smart cameras or access control systems, etc.

\section{Conclusion}
This work observed that pruning, while effective in compressing large models with minimal loss of accuracy, can result in substantial disparate accuracy impacts. The paper examined the factors causing such disparities both theoretically and empirically showing that: {\bf (1)} disparity in gradient norms across groups and {\bf (2)} disparity in Hessian matrices associated with the loss functions computed using a groups' data are two key factors responsible for such disparities to arise.
By recognizing these factors, the paper also developed a simple yet effective retraining technique that largely mitigates the disparate impacts caused by pruning.

As reduced versions of large, overparameterized models become increasingly adopted in embedded systems to facilitate autonomous decisions, we believe that this work makes an important step toward {\em understanding} and {\em mitigating} the sources of disparate impacts observed in compressed learning models and hope it will spark awareness in this important area.

\section*{Acknowledgments}
This research is partly funded by NSF grants SaTC-1945541, SaTC-2133169, and CAREER-2143706. 
F.~Fioretto is also supported by a Google Research Scholar Award and an Amazon Research Award. 
The views and conclusions of this work are those of the authors only.

\bibliographystyle{abbrvnat}
\bibliography{lib}
\newpage
\section*{NeurIPS 2022 Paper Checklist}



\newcommand{\ans}[1]{{\color{purple}{{\em #1}}}}

\begin{enumerate}
\item For all authors
\begin{enumerate}[label=\alph*)]
	\item Do the main claims made in the abstract and introduction accurately reflect the paper's contributions and scope?

	\ans{Yes. The paper contributions are stated in the abstract and listed in the Introduction.}

	\item (b) Have you read the ethics review guidelines and ensured that your paper conforms to them?

	\ans{Yes.}

	\item Did you discuss any potential negative societal impacts of your work?
	\ans{This work sheds light on the reasons behind the observed disparate impacts and fairness violations through pruning. Pruning is a widely-used compression technique for large-scale models which are deployed in settings with less resources. Thus, the insights generated by this work may have a positive societal impact.}

	\item Did you describe the limitations of your work?




	\ans{Yes. Please, see section \ref{sec:limitations}.}
\end{enumerate}

\item If you are including theoretical results...
\begin{enumerate}
	\item Did you state the full set of assumptions of all theoretical results?

	\ans{Yes. The assumptions were stated in or before each Theorem and also reported in the Appendix.} 

	\item Did you include complete proofs of all theoretical results?


	\ans{Yes. While the main paper only contains proof sketches or intuitions, all complete proofs are reported in Appendix \ref{sec:missing_proofs}.}
\end{enumerate}

\item If you ran experiments...
\begin{enumerate}
	\item Did you include the code, data, and instructions needed to reproduce the main experimental results (either in the supplemental material or as a URL)?




	\ans{Yes. Code, datasets and experiments were submitted in 
	the supplemental material. We also provide a  link in Appendix \ref{sec:additional_results}.}

	\item Did you specify all the training details (e.g., data splits, hyperparameters, how they were chosen)?

	\ans{Yes. See Appendix  \ref{sec:additional_results}.}

	\item Did you report error bars (e.g., with respect to the random seed after running experiments multiple times)?

	\ans{The main evaluation metric adopted in this work is the excessive loss (see Equations \eqref{eq:2} and \eqref{eq:3})
	which implicitly captures the randomness of the private mechanisms. Providing error bars would be misleading.}

	\item Did you include the amount of compute and the type of resources used (e.g., type of GPUs, internal cluster, or cloud provider)?




	\ans{Yes. See Appendix \ref{sec:additional_results}}
\end{enumerate}

\item If you are using existing assets (e.g., code, data, models) or curating/releasing new assets...
\begin{enumerate}
	\item If your work uses existing assets, did you cite the creators?

	\ans{Yes. See References section.}

	\item Did you mention the license of the assets?





	\ans{Yes, when available.}

	\item Did you include any new assets either in the supplemental material or as a URL?


	\ans{No new asset was required to perform this research.}

	\item Did you discuss whether and how consent was obtained from people whose data you're using/curating?

	\ans{Yes. The paper uses public datasets.}

	\item Did you discuss whether the data you are using/curating contains personally identifiable information or offensive content?


	\ans{No. The adopted data is composed of standard benchmarks that have been used extensively in the ML literature and we believe the above does not apply.}
\end{enumerate}
	\item If you used crowdsourcing or conducted research with human subjects...

	\ans{No. This research did not use crowdsourcing.}
\end{enumerate}

\newpage
\appendix


\setcounter{theorem}{0}
\setcounter{lemma}{0}
\setcounter{proposition}{0}
\setcounter{property}{0}

\section{Missing Proofs}
\label{sec:missing_proofs}

\begin{theorem}
\label{thm:taylor} 
The \emph{excessive loss} of a group $a \in \cA$ is upper bounded by\footnote{
  With a slight abuse of notation, the results refer to $\bar{\btheta}$ as the homonymous vector which is extended with $k-\bar{k}$ zeros.
}: 
\begin{align}
  R(a)  \leq 
  \left\| \bm{g}^{\ell}_a \right\|
  \times \left\|  \bar{\btheta} - \opttheta\right\|
  + 
  \frac{1}{2} \lambda \left( \bm{H}_{a}^{\ell} \right) 
  \times 
  \left\| \bar{\btheta} - \opttheta\right\|^2 
  + 
  \cO\left( \left\|\bar{\btheta} - \opttheta \right\|^3 \right),
  \label{eq:thm1}
\end{align}
where 
$\bm{g}^{\ell}_a = \nabla_{\opttheta} J( \opttheta; D_{a})$ is the vector of gradients associated with the loss function $\ell$ evaluated at $\opttheta$ and computed using group data $D_a$,  
$\bm{H}_{a}^{\ell} = 
\nabla^2_{\opttheta} J(\opttheta; D_{a})$ is the Hessian matrix of the loss function $\ell$, at the optimal parameters vector $\opttheta$, computed using the group data $D_a$ (henceforth simply referred to as \emph{group hessian}), and  
$\lambda(\Sigma)$ is the maximum eigenvalue of a matrix $\Sigma$.
\end{theorem}


\begin{proof}
Using a second order Taylor expansion around $\opttheta$, the excessive loss $R(a)$ for a group $a \in \cA$ can be stated as:
\begin{align*}
  R(a) 
  &=  J(\bar{\btheta}; D_{a}) - J(\opttheta; D_{a}) \\  
  & =   \left[ J\left(\opttheta; D_{a}\right) 
    +  \left(\bar{\btheta} - \opttheta \right)^\top \, 
    \nabla_{\theta} J\left(\opttheta; D_{a} \right)  
  + \frac{1}{2} \left(\bar{\btheta} - \opttheta \right)^\top\, 
  \bm{H}_{a}^\ell \left(\bar{\btheta} - \opttheta \right) 
   + \cO \left( \left\| \opttheta - \bar{\btheta} \right\|^3 \right) 
   \right] 
   - J\left(\opttheta; D_{a}\right)  \\
  &= \left(\bar{\btheta} - \opttheta \right)^\top\,  \bm{g}^{\ell}_a 
  + \frac{1}{2} \left(\bar{\btheta} - \opttheta \right)^\top \,
  \bm{H}_a^\ell \left(\bar{\btheta} - \opttheta \right) 
  +  \cO \left( \|\opttheta - \bar{\btheta}\|^3 \right)
\end{align*}
The above, follows from the loss $\ell(\cdot)$ being at least twice differentiable, by assumption.

By Cauchy-Schwarz inequality, it follows that 
\[
    \left( \bar{\btheta} - \opttheta \right)^\top 
    \bm{g}^{\ell}_a \leq 
    \left\| \bar{\btheta} - \opttheta \right\| \times 
    \left\|\bm{g}^{\ell}_a \right\|.
\] 
In addition, due to the property of Rayleigh quotient we have:
\[
    \frac{1}{2} \left(\bar{\btheta} - \opttheta\right)^\top 
    \bm{H}_{a}^\ell \left(\bar{\btheta} - \opttheta\right) 
    \leq 
    \frac{1}{2} \lambda \left(\bm{H}_a^\ell \right) \times 
    \left\| \bar{\btheta} - \opttheta \right\|^2.
\]
The upper bound for the excessive loss $R(a)$ is thus obtained by combining these two inequalities.
\end{proof}

\begin{proposition}
\label{thm:grad_imbalance} 
Consider two groups $a$ and $b$ in $\cA$ with $|D_a| \geq |D_b|$. Then 
\(
\left\| \bm{g}_a^\ell \right\|  \leq  \left\| \bm{g}_b^\ell \right\|.
\)
\end{proposition}

\begin{proof}
By the assumption that the model converges to a  local minima, it
follows that:

\begin{align*} 
    \nabla_{\theta} \cL(\opttheta; D) 
    &= \sum_{a \in \cA} \frac{|D_a|}{|D|} \nabla_{\theta} 
    J(\opttheta; D_a) \\
    &= \frac{|D_a|}{|D|} \bm{g}^{\ell}_a + \frac{|D_b|}{|D|} 
    \bm{g}^{\ell}_b  
    = \bm{0}
\label{eq:grad_decomp}
\end{align*}

Thus, $\bm{g}^{\ell}_a = -\frac{|D_b|}{|D_a|} \bm{g}_b$. Hence $\| \bm
{g}^{\ell}_a \| = \frac{|D_b|}{|D_a|} \| \bm{g}^\ell_b \|  \leq \| \bm{g}^\ell_b\| $, because
$|D_a| \geq |D_b| $.
\end{proof}

\begin{proposition}
\label{thm:acc_vs_grad}
For a given group $a \in \cA$, gradient norms can be upper bounded as:
\[
    \|\bm{g}_a^\ell \| \in
    \cO\left( 
    \sum_{(\bm{x}, y) \in D_a} 
    \underbrace{\|f_{\opttheta}(\bm{x}) - y \|}_{\textit{Accuracy}}
    \times
    \left\| \nabla_{\opttheta} f_{\opttheta}(\bm{x}) \right\|
    \right).
\]
\end{proposition}

The above proposition is presented in the context of cross entropy loss or mean squared error loss functions. These two cases are reviewed as follows

\paragraph{Cross Entropy Loss.} 
Consider a classification task with cross entropy loss: 
$\ell(f_{\opttheta}(\bm{x}), y) = - \sum_{z \in \cY} f^z_{\opttheta}
(\bm{x}) \bm{y}^z$, where $f_{\opttheta}^z(\bm{x})$ represents the $z$-th element of the output associated with the soft-max 
layer of model $f_{\opttheta}$, and $\bm{y}$ is a one-hot encoding of the true label $y$, with $\bm{y}^z$ representing its $z$-th element, then,
\begin{align*}
    \| \bm{g}_a\|  
    = \left\| \nabla_{\btheta} J(\opttheta; D_a,) \right\| 
    &= \left\| \nicefrac{1}{|D_a|} \sum_{(\bm{x},y) \in D_{a}} 
    \nabla_{f} \ell(f_{\opttheta}(\bm{x}), y) \times 
    \nabla_{\btheta} f_{\opttheta}(\bm{x}) \right\|\\
    &= \left\| \nicefrac{1}{|D_a|} \sum_{(\bm{x}, y) \in D_a} 
    (f_{\opttheta}(\bm{x}) - \bm{y}) \times  
    \nabla_{\btheta} f_{\opttheta}(\bm{x}) \right\|\\
    &\leq \nicefrac{1}{|D_a|}\sum_{(\bm{x}, y) \in D_a}  
    \left\| f_{\opttheta}(\bm{x}) - \bm{y}  \right\| \times 
    \left\| \nabla_{\btheta} f_{\opttheta}(\bm{x}) \right\|,
\end{align*}
where the third equality is due to that the gradient of the cross entropy loss reduces to $f_{\opttheta}(\bm{x}) - \bm{y}$. 

\paragraph{Mean Squared Error.} 
Next, consider a regression  task with mean squared error loss 
$\ell(f_{\opttheta}(\bm{x}), y) =  ( f_{\opttheta}(\bm{x}) - y)^2$. Using the same notation as that made above, if follows:
\begin{align*}
    \| \bm{g}_a\|  
    = \left\| \nabla_{\btheta} J(\opttheta; D_a,) \right\| 
    &= \left\| \nicefrac{1}{|D_a|} \sum_{(\bm{x},y) \in D_{a}} 
    \nabla_{f} \ell(f_{\opttheta}(\bm{x}), y) \times 
    \nabla_{\btheta} f_{\opttheta}(\bm{x}) \right\|\\
    &= \left\| \nicefrac{2}{|D_a|} \sum_{(\bm{x}, y) \in D_a} 
    (f_{\opttheta}(\bm{x}) - \bm{y}) \times  
    \nabla_{\btheta} f_{\opttheta}(\bm{x}) \right\|\\
    &\leq \nicefrac{2}{|D_a|}\sum_{(\bm{x}, y) \in D_a}  
    \left\| f_{\opttheta}(\bm{x}) - \bm{y}  \right\| \times 
    \left\| \nabla_{\btheta} f_{\opttheta}(\bm{x}) \right\|,
\end{align*}
where the third equality is due to that the gradient of the mean squared error loss w.r.t.~$f_{\opttheta}(\cdot)$ reduces to $2( f_{\opttheta}(\bm{x}) - \bm{y})$. 




\begin{theorem}
\label{thm:hessian_norm_bound} 
Let $f_\theta$ be a binary classifier trained using a binary cross entropy loss. For any group $a \in \cA$, the maximum eigenvalue of the group Hessian $\lambda(\bm{H}_a^{\ell})$ can be upper bounded by:
\begin{equation}
    \lambda(\bm{H}_a^{\ell}) \leq \frac{1}{|D_a|}
    \sum_{(\bm{x}, y)\in D_a}
    \underbrace{
        \left( {f}_{\opttheta}(\bm{x}) \right) 
        \left( 1 - {f}_{\opttheta}(\bm{x}) \right)}_{\textit{Distance to decision boundary}} 
        \times 
        \left\| \nabla_{\btheta} f_{\opttheta}(\bm {x}) \right\|^2 
     + 
         \underbrace{\left| f_{\opttheta}(\bm{x}) - y \right|}_{\textit{Accuracy}} 
         \times 
        \lambda\left( \nabla^2_{\btheta} f_{\opttheta}(\bm{x}) \right). 
    \label{eq:hessian_norm_bound}
\end{equation}
\end{theorem}

\begin{proof}
First notice that an upper bound for the Hessian loss computed on a group $a \in \cA$ can be derived as:
\begin{align}
    \lambda(\bm{H}_a^{\ell}) 
    &= \lambda\left( 
    \frac{1}{|D_a|} \sum_{(\bm{x}, y) \in D_a} \bm{H}_{\bm{x}}^{\ell} \right) 
    \leq 
    \frac{1}{|D_a|} \sum_{(\bm{x}, y) \in D_a}  
    \lambda\left( \bm{H}_{\bm{x}}^{\ell} \right)
  \label{eq:hessian_decompose}
\end{align}
where $\bm{H}_{\bm{x}}^{\ell}$ represents the Hessian loss associated with a sample  $\bm{x} \in D_a$ from group $a$.
The above follows Weily's inequality which states that for any two symmetric matrices $A$ and $B$,  $\lambda(A + B) \leq \lambda(A) + \lambda(B)$. 

Next, we will derive an upper bound on the Hessian loss associated to a sample $\bm{x}$. First, based on the chain rule a closed form expression for the Hessian loss associated to a sample $\bm{x}$ can be written as follows:
\begin{align} 
    \bm{H}^{\ell}_{\bm{x}} &=  
    \nabla^2_{f}   \ell\left(f_{\opttheta}(\bm{x}), y\right)
    \left[ \nabla_{\btheta} f_{\opttheta}(\bm{x}) 
    \left( \nabla_{\btheta} f_{\opttheta}(\bm{x}) \right)^\top \right]
    + 
    \nabla_{f} \ell\left(f_{\opttheta}(\bm{x}), y\right) 
    \nabla^2_{\btheta} f_{\opttheta}(\bm{x}).
   \label{eq:hessian_sample}
\end{align}
The next follows from that 
\begin{align*}
\nabla_{f}   \ell\left( f_{\opttheta}(\bm{x}), y\right) 
&= (f_{\opttheta}(\bm{x}) - y), \\
\nabla^2_{f}   \ell\left(f_{\opttheta}(\bm{x}), y\right) 
&= f_{\opttheta}(\bm{x}) \left(1 - f_{\opttheta}(\bm{x})\right).
\end{align*}
Applying the Weily inequality again on the R.H.S.~of Equation \ref{eq:hessian_sample}, we obtain:
\begin{align}
    \lambda(\bm{H}_{\bm{x}}^{\ell}) & \notag
    \leq  f_{\opttheta}(\bm{x}) 
    \left(1 - f_{\opttheta}(\bm{x})\right) \times 
    \left\| \nabla_{\btheta} f_{\opttheta}(\bm{x}) \right\|^2 
    + \lambda\left( f_{\opttheta}(\bm{x}) - y\right) \times 
    \nabla^2_{\btheta} f_{\opttheta}(\bm{x})\\
    & \leq 
    f_{\opttheta}(\bm{x}) 
    \left(1 - f_{\opttheta}(\bm{x})\right) \times 
    \left\| \nabla_{\btheta} f_{\opttheta}(\bm{x}) \right\|^2 
    + 
    \left| f_{\opttheta}(\bm{x}) - y \right| 
    \lambda \left( \nabla^2_{\btheta} 
    f_{\opttheta}(\bm{x}) \right)
    \label{eq:hessian_bound}
\end{align}

The statement of Theorem \ref{thm:hessian_norm_bound} is obtained combining Equations \ref{eq:hessian_bound} with \ref{eq:hessian_decompose}.
\end{proof}

\begin{proposition}
\label{prop:dist_bndry}
 Consider a binary  classifier $f_{\btheta}(\bm{x}) $. For a given sample $\bm{x} \in D $, the term ${f}_{\opttheta}(\bm{x}) (1 - {f}_{\opttheta}(\bm{x}) )$ is maximized when ${f}_{\opttheta}(\bm{x})  = 0.5$ and minimized when ${f}_{\opttheta}(\bm{x}) \in \{0,1\} $.
\end{proposition}

\begin{proof}

First, notice that $f_{\opttheta}(\bm{x}) \in [0,1]$, as it represents the soft prediction (that returned by the last layer of the network), thus ${f}_{\opttheta}(\bm{x}) \geq f^2_{\opttheta}(\bm{x})$. It follows that:

\begin{equation}
    f_{\opttheta}(\bm{x}) \left( 1 -  f_{\opttheta}(\bm{x}) \right) 
    =  f_{\opttheta}(\bm{x}) - f^2_{\opttheta}(\bm{x}) \geq 0.
\end{equation}
In the above, it is easy to observe that the equality holds when either $f_{\opttheta}(\bm{x}) = 0$ or $f_{\opttheta}(\bm{x}) = 1$.

Next, by the Jensen inequality, it follows that:
\begin{equation}
    f_{\opttheta}(\bm{x}) \left( 1 -  f_{\opttheta}(\bm{x}) \right) 
    \leq  \frac{\left( f_{\opttheta}(\bm{x})+1 
                     - f_{\opttheta}(\bm{x}) \right)^2}{4} 
    = \frac{1}{4}.
\end{equation}
The above holds when $f_{\opttheta}(\bm{x}) = 1 - f_{\opttheta}(\bm{x})$, in other words, when $f_{\opttheta}(\bm{x}) = 0.5$. 
Notice that, in the case of binary classifier, this refers to the case when the sample $\bm{x}$ lies on the decision boundary. 
\end{proof}

\section{Dataset and Experimental Settings}
\label{sec:additional_results}

\subsection{Datasets}

The paper uses the following datasets to validate the findings discussed in the main paper: 
\begin{itemize}
    \item {\bf UTK-Face}~\cite{zhifei2017cvpr}.
    A large-scale face dataset with a long age span (range from 0 to 116 years old). The dataset consists of over 20,000 face images with annotations of age, gender, and ethnicity. The images cover large variations in pose, facial expression, illumination, occlusion, resolution, etc. 
    The experiments adopt the following attributes for classification (e.g., $\cY$) and as protected group ($\cA$):  {\sl ethnicity}, {\sl age bins}, {\sl gender}.
    
    \item {\bf CIFAR-10}~\cite{cifar10}.
    This dataset consists of 60,000 32$\times$32 RGB images in 10 classes, with 6,000 images per class. The 10 different classes represent airplanes, cars, birds, cats, deer, dogs, frogs, horses, ships, and trucks. 
    
    \item {\bf SVHN}~\cite{Netzer2011ReadingDI}
    Street View House Numbers (SVHN) is a digit classification dataset that contains 600,000 32$\times$32 RGB images of printed digits (from 0 to 9) cropped from pictures of house number plates.
    

\end{itemize}

\subsection{Architectures, Hyper-parameters, and Settings}

The study adopts the following architectures to validate the results of the main paper:
\begin{itemize}
    \item {\bf ResNet18}: An 18-layer architecture, with 8 residual blocks. Each residual block consists of two convolution layers.  The model has $\sim 11$ million trainable parameters. 
    \item {\bf ResNet50} This model contains 48 convolution layers, 1 MaxPool layer and a AvgPool layer. ResNet50 has $\sim 25$ million trainable parameters.
    \item {\bf VGG-19} This model consists of 19 layers (16 convolution layers, 3 fully connected layers, 5 MaxPool layers and 1 SoftMax layer). The model has $\sim 143$ million parameters.  
\end{itemize}

Hyperparameters for each of the above models was performed over a grid search (for different learning rates = $[0.0001, 0.001, 0.01, 0.1, 0.5, 0.05, 0.005, 0.0005]$) over a cluster of NVIDIA RTX A6000 with the above networks using the UTKFace dataset. 
The models with the highest accuracy were chosen and employed for the assessment of the mitigation solution in Sec.~\ref{sec:AssmtMitigation}. The running time required for all sets of experiments which include mitigation solutions was about \textasciitilde{3} \emph{days}. 

The training uses the SGD optimizer  with a momentum  of $0.9$ and weight\_decay of $1e^{-4}$.  
Finally, the Lagrangian step size adopted in the Lagrangian dual learning framework \cite{fioretto2020lagrangian} is set to $0.001$.

All the models developed were implemented using Pytorch 3.0. The training was performed using NVidia Tesla P100-PCIE-16GB GPUs and 2GHz Intel Cores. The model was run for 100 epochs for the CIFAR-10 and SVHN and 40 epochs for UTK-Face dataset. 
Each reported experiment is an average of 10 repetitions. 
In all experiments, the protected group set coincides with the target label set: i.e., $\cA = \cY$.


\section{Additional Experimental Results}
\label{sec:additional_results}



\subsection{Impact of pruning on fairness}

This section shows and affirms the impact of pruning towards accuracy disparity through VGG-19 network. The same training procedures as employed with ResNet18 in Fig~\ref{fig:motivation} were followed. Each demographic group's accuracy is shown before and after pruning on the UTK-Face dataset in two cases: when \emph{ethnicity} is a group attribute as in Figure \ref{fig:utk_race_vgg_motivation}, and when \emph{gender} is a group attribute as in Figure \ref{fig:utk_gender_vgg_motivation}. A consistent message is that under a higher pruning rate, the accuracies are more imbalanced across different groups, emphasizing the negative impact of pruning on fairness.


\begin{figure*}[t]
 \centering
    \includegraphics[width=\linewidth]{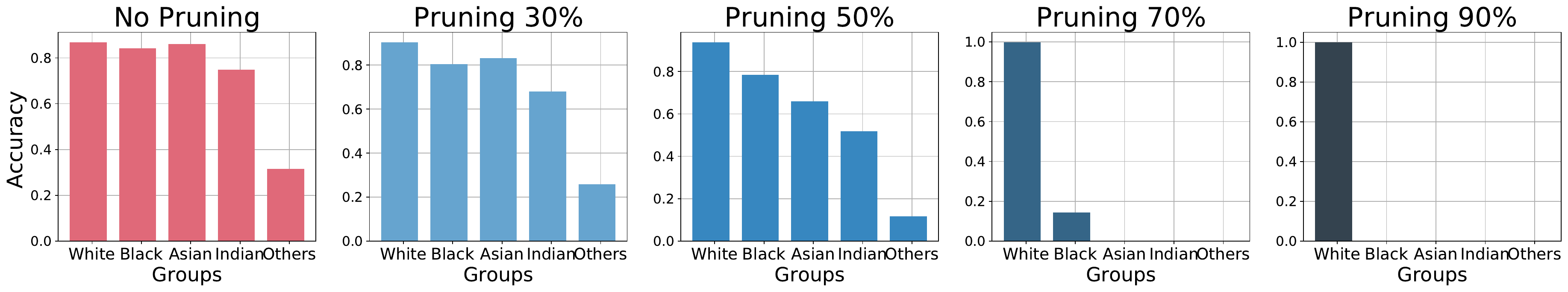}
    \caption{Accuracy of each demographic group in the UTK-Face dataset with ethnicity (5 classes) as group attribute using VGG19 over increasing pruning rates.}
   \label{fig:utk_race_vgg_motivation}
\end{figure*}

\begin{figure*}[t]
 \centering
    \includegraphics[width=\linewidth]{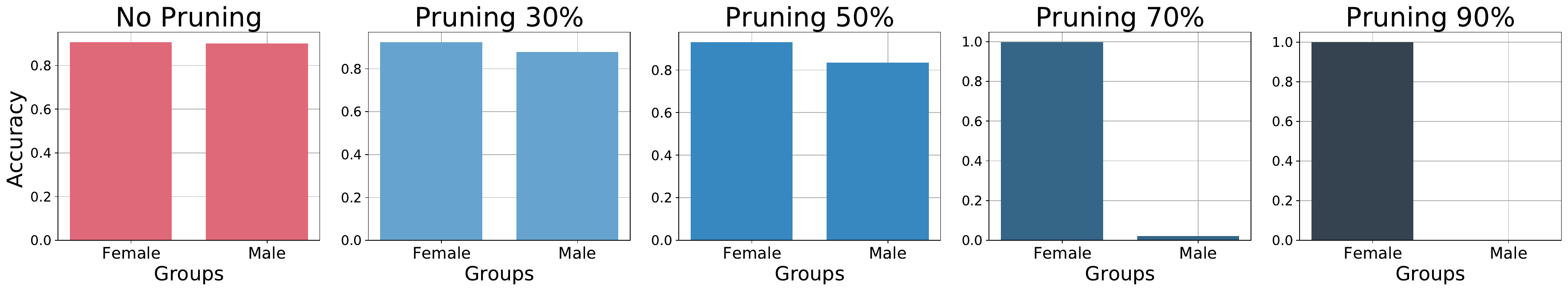}
    \caption{Accuracy of each demographic group in the UTK-Face dataset with gender (2 classes) as group attribute using VGG19 over increasing pruning rates.}
   \label{fig:utk_gender_vgg_motivation}
\end{figure*}


\subsection{Correlation of gradient/hessian norm and average distance to the decision boundary}

This subsection elaborates the impact of gradient norms and group Hessians towards the fairness issues shown in Figures \ref{fig:utk_race_vgg_motivation} and  \ref{fig:utk_gender_vgg_motivation}. In Section \ref{sec:grad_analysis}, it has been shown that the group with a larger gradient norm before pruning will be penalized more than the groups with a smaller gradient norm. Figures  \ref{fig:utk_race_vgg_grad} and \ref{fig:utk_gender_vgg_grad} show the gradient norm of each demographic group for UTK-Face dataset under two choices of protected attributes for VGG 19 networks. The results indicate that a group penalized less will have a smaller gradient norm with respect to those of the other groups. 

In addition, Section \ref{sec:hessian_analysis} supports that Hessian norm is another factor. More precisely, the groups with a larger Hessian norm will be penalized more (drop much more in accuracy) than groups with a smaller Hessian norm. Evidence is provided for the claim on VGG19 in Figures \ref{fig:utk_gender_vgg_grad} and \ref{fig:utk_race_vgg_grad}. These results on VGG19 again confirm the theoretical findings.

Finally, in Section \ref{sec:hessian_analysis}, a positive correlation between gradient norms and Hessian groups is shown in Theorem \ref{thm:hessian_norm_bound}, and a negative correlation between Hessian groups and distance to the decision boundary is shown in Proposition \ref{prop:dist_bndry}. These important results again are supported by the results in Figures \ref{fig:utk_gender_vgg_grad} and \ref{fig:utk_race_vgg_grad}.

\begin{figure*}[t]
 \centering
    \includegraphics[width=0.8\linewidth]{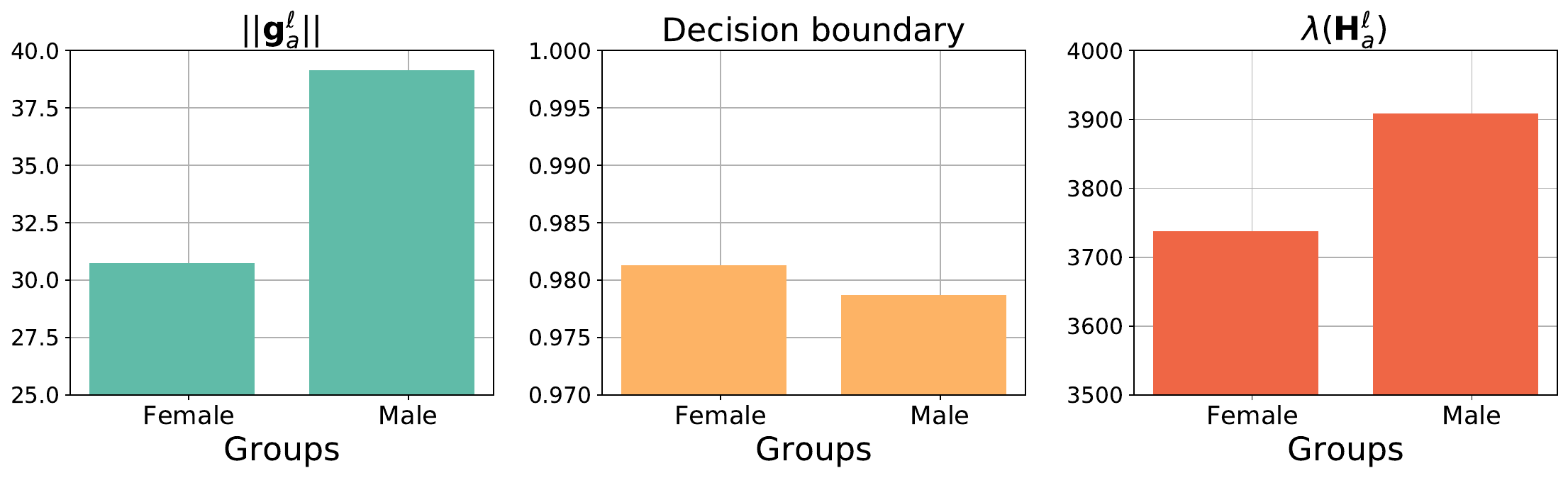}
    \caption{Gradient/Hessian norm and average distance to the decision boundary of each demographic group in the UTK-Face dataset with gender (2 classes) as group attribute using VGG19 with no pruning.}
   \label{fig:utk_gender_vgg_grad}
\end{figure*}

\begin{figure*}[t]
 \centering
    \includegraphics[width=0.8\linewidth,height=100pt]{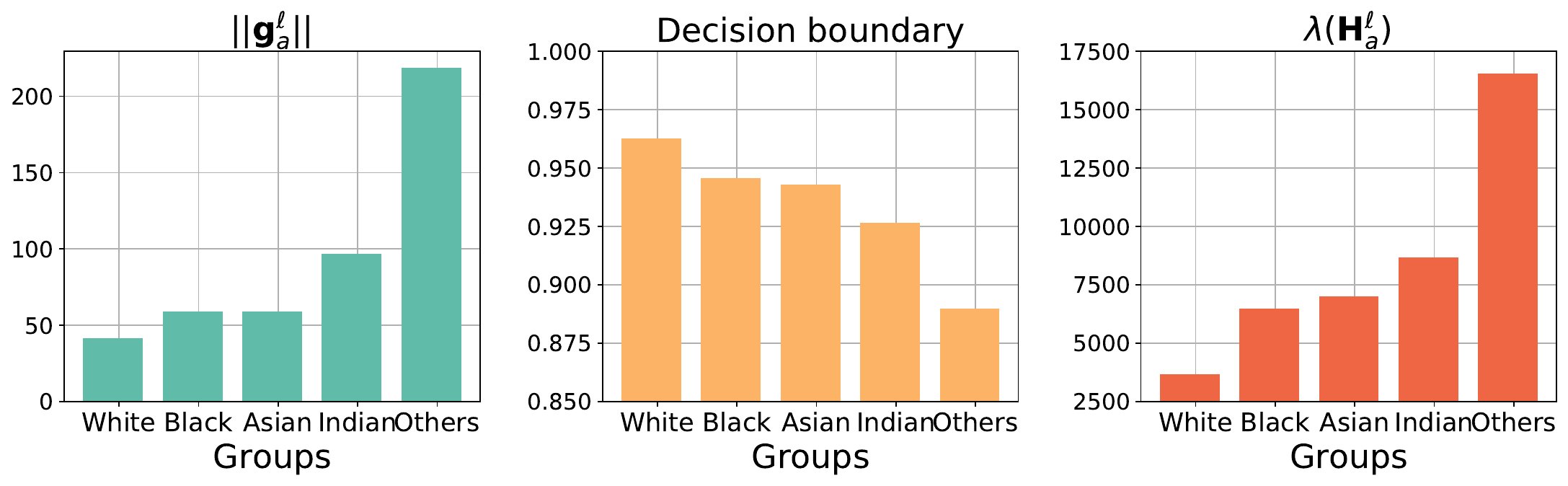}
    \caption{Gradient/Hessian norm and average distance to the decision boundary of each demographic group in the UTK-Face dataset with ethnicity (5 classes) as group attribute using VGG19 with no pruning.}
   \label{fig:utk_race_vgg_grad}
\end{figure*}

\subsection{Impact of group sizes to gradient norm}

This section presents additional empirical results to support Theorem \ref{thm:grad_imbalance}, stating that the group with more samples will tend to have a smaller gradient norm. In these  experiments, run on a ResNet50 network, one group is chosen and upsampled 1$\times$, 5$\times$, 10$\times$, and 20$\times$ times. Note that by increasingly upsampling it, the group becomes the majority group in that dataset. A group with \emph{more samples} is expected to end up with a \emph{smaller gradient norm} when the training convergences.

\paragraph{UTK-Face with gender}
Since the UTK-Face is balanced with regard to gender (Female/Male), the number of samples in Female, and Male groups is upsampled in turn. Figure \ref{fig:utk_gender_group_sizes_impact} reports the respective gradient norms at the last training iteration when upsampling Females (left) and Males (right.) Note how the Male group, initially with no upsampling, has a larger gradient norm than the Female group (right sub-plot). However, if the number of Male samples is increased enough, its gradient norm becomes smaller than that of the Female group.

 \begin{figure*}[t]
 \centering
\includegraphics[width=0.6\linewidth]{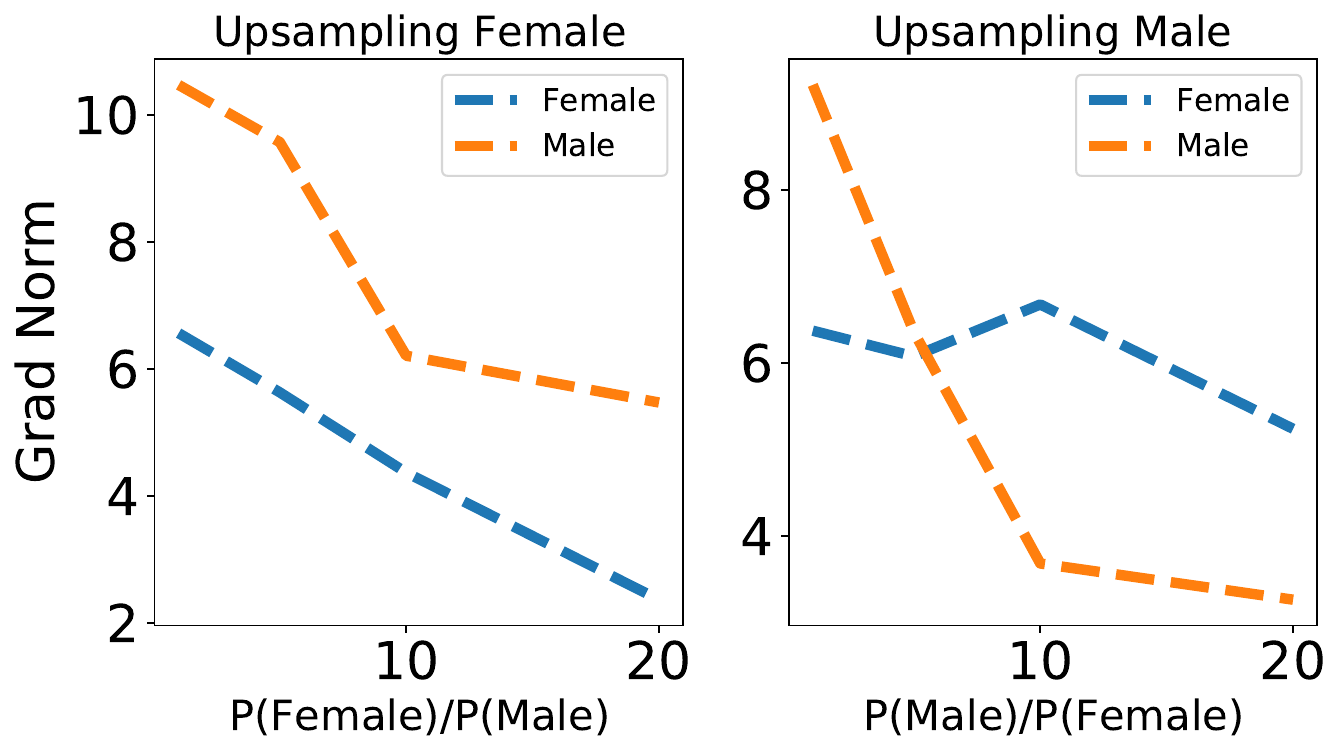}
\caption{Impact of group sizes to the gradient norm per group in UTK-Face datase where groups are Male and Female.}
 \label{fig:utk_gender_group_sizes_impact}
\end{figure*}

\paragraph{UTK-Face with age bins}
Similar experiments are performed with UTK-Face on nine age bin groups. Three age bins are randomly chosen, 0, 2, 4, and the number of samples for each group is upsampled in turn. The gradient norms of nine age bin groups are shown in Figure \ref{fig:utk_age_bins_group_sizes}, where the upsampled groups are highlighted with dotted thick lines. The results echo that if a group's number of samples is increased enough, its gradient norm at convergence will be smaller than the other 8 age bin groups. 

 \begin{figure*}[t]
 \centering
\includegraphics[width=0.8\linewidth]{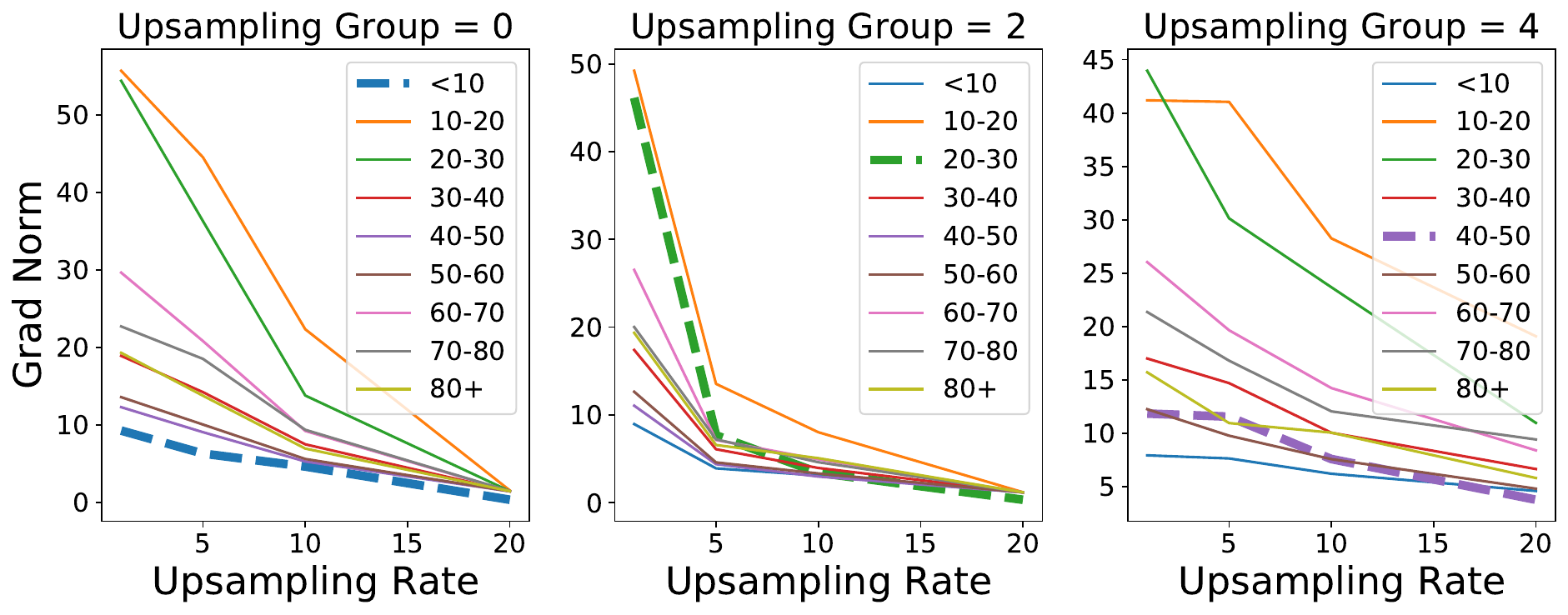}
\caption{Impact of group sizes to the gradient norm per group in UTK-Face dataset where groups are nine age bins. The group with dotted thick line is a \emph{majority group} in each chart.}
 \label{fig:utk_age_bins_group_sizes}
\end{figure*}








\end{document}